\documentclass{article}

\usepackage{arxiv}

\usepackage{amsfonts}
\usepackage{graphicx}
\usepackage{booktabs}
\usepackage{standalone}
\usepackage[normalem]{ulem}
\usepackage{enumerate}
\usepackage{paralist}
\usepackage{multirow}
\usepackage{caption}
\usepackage{subcaption}
\usepackage[separate-uncertainty=true]{siunitx}
\usepackage[english]{babel}
\usepackage{xr}
\usepackage{algorithm}
\usepackage{algorithmicx}
\usepackage{algpseudocode}
\usepackage{wrapfig}
\usepackage{listings}
\usepackage{tikz}
\usetikzlibrary{positioning, fit, arrows.meta}

\usepackage[round]{natbib}
\newcommand{\ele}[2][]{\ell_{#1}({#2})}
\newcommand{\eletilde}[2][]{\widetilde{\ell}_{#1}({#2})}

\graphicspath{{figs/}} %
\usepackage{xcolor}

\usepackage{amsfonts}
\usepackage{amsmath}
\usepackage{amsthm}
\usepackage{mathtools}
\usepackage{ulem}
\usepackage{cancel}

\usepackage{ifthen}

\renewcommand{\phi}{\varphi}
\renewcommand{\epsilon}{\varepsilon}

\newcommand{\R}{\mathbb{R}}
\newcommand{\X}{{X}}
\newcommand{\Z}{{Z}}

\renewcommand{\L}{{L}}

\newcommand{\z}{{z}}
\newcommand{\x}{{x}}
\newcommand{\xtilde}{\widetilde{\x}}
\newcommand{\xline}{\overline{\x}}
\newcommand{\etatilde}{{\widetilde{\eta}}}
\newcommand{\Ltilde}{\widetilde{L}}

\newcommand{\Xb}{{\boldsymbol{\X}}}
\newcommand{\Zb}{{\boldsymbol{\Z}}}
\newcommand{\xb}{{\boldsymbol{\x}}}
\newcommand{\zb}{{\boldsymbol{\z}}}
\newcommand{\betab}{{\boldsymbol{\beta}}}
\newcommand{\phib}{{\boldsymbol{\phi}}}
\newcommand{\gammab}{{\boldsymbol{\gamma}}}

\newcommand{\etab}{{\boldsymbol{\eta}}}
\newcommand{\xtildeb}{{\widetilde{\xb}}}

\newcommand{\etatildeb}{{{\widetilde{\etab}}}}

\newcommand{\elbo}{\mathcal{L}}
\newcommand{\std}{{\operatorname{std}}}

\DeclareMathOperator{\Normal}{\mathcal{N}}

\DeclarePairedDelimiterX{\infdivx}[2]{(}{)}{%
	#1 \delimsize\|\, #2%
}

\DeclareMathOperator{\KLoperator}{KL}
\newcommand{\KL}[2]{\KLoperator\infdivx{#1}{#2}}

\DeclareMathOperator{\Eoperator}{\mathbb{E}}
\newcommand{\E}[2][]{\Eoperator_{#1}\left[#2\right]}

\newcommand{\seq}[3]{\left\{#1_{#2}\right\}_{\ifx&#3&\else #2=1\fi}^{#3}}  %
\newcommand{\range}[2][]{{\ifx&#1&1, 2, \dots, #2\else#1_1, #1_2,\dots, #1_{#2}\fi}}

\DeclarePairedDelimiter\abs{\lvert}{\rvert}%
\newcommand{\norm}[2][]{{\left|\left|#2\right|\right|}_{#1}}

\DeclareMathOperator*{\argmin}{argmin}

\newcommand{\numberthis}{\addtocounter{equation}{1}\tag{\theequation}}

\newcommand{\spd}[2][]{\partial^{#1}_{#2}}
\newcommand*\dif{\mathop{}\!\mathrm{d}}

\newtheorem{theorem}{Theorem}[section]
\newtheorem{proposition}{Proposition}[section]

\theoremstyle{definition}

\usepackage{amsthm}

\makeatletter
\renewenvironment{proof}[1][\proofname]{%
	\par\pushQED{\qed}\normalfont%
	\topsep6\p@\@plus6\p@\relax
	\trivlist\item[\hskip\labelsep\bfseries#1\@addpunct{.}]%
	\ignorespaces
}{%
	\popQED\endtrivlist\@endpefalse
}
\makeatother

\usepackage{pifont}
\newcommand{\minusmark}{-}%

\newcommand{\red}[1]{{\color{red} #1}}
\newcommand{\green}[1]{{\color{green!50!black!100} #1}}

\usepackage{xspace}

\newcommand{\needcite}[1][]{\textbf{\color{red} [cite\ifthenelse{\equal{#1}{}}{}{~{#1}}]}\xspace}

\newcommand{\edit}[2][]{{\textcolor{orange}{\sout{#2}}}\ifthenelse{\equal{#1}{}}{}{\ifthenelse{\equal{#2}{}}{}{$\rightarrow$}{\textcolor{green!40!black!90}{#1}}}}

\usepackage{hyperref}
\definecolor{gt}{rgb}{0.4, 0.4, 0.4}
\definecolor{std}{rgb}{0.9019607843137255, 0.6705882352941176, 0.00784313725490196}
\definecolor{std-gamma}{rgb}{0.4588235294117647, 0.4392156862745098, 0.7019607843137254}
\definecolor{ours}{rgb}{0.4, 0.6509803921568628, 0.11764705882352941}
\definecolor{max}{rgb}{0.9, 0.1, 0}
\definecolor{iqr}{rgb}{1, 0.5, 0}
\definecolor{ours-none}{rgb}{0.10588235294117647, 0.6196078431372549, 0.4666666666666667}
\definecolor{ours-bern}{rgb}{0.9058823529411765, 0.1607843137254902, 0.5411764705882353}

\title{Lipschitz standardization for multivariate learning}

\author{%
	Adri\'an Javaloy\\ %
	Probabilistic Learning Group\\
	MPI for Intelligent Systems\\
	Tübingen, Germany \\
	\texttt{ajavaloy@tue.mpg.de} \\
	\And
	Isabel Valera\\
	Probabilistic Learning Group\\
	MPI for Intelligent Systems\\
	Tübingen, Germany \\
	\texttt{ivalera@tue.mpg.de}
}

\newcommand{\lipgamma}{\emph{\textbf{\textcolor{ours}{lip-gamma}}}\xspace}
\newcommand{\stdnone}{\emph{\textbf{\textcolor{std}{std-none}}}\xspace}
\newcommand{\lipnone}{\emph{\textbf{\textcolor{ours-none}{lip-none}}}\xspace}
\newcommand{\lipbern}{\emph{\textbf{\textcolor{ours-bern}{lip-bern}}}\xspace}
\newcommand{\stdgamma}{\emph{\textbf{\textcolor{std-gamma}{std-gamma}}}\xspace}
\newcommand{\maxnone}{\emph{\textbf{\textcolor{max}{max-none}}}\xspace}
\newcommand{\iqrnone}{\emph{\textbf{\textcolor{iqr}{iqr-none}}}\xspace}

\begin{document}

	\maketitle

\begin{abstract}

Probabilistic learning is increasingly being tackled as an optimization problem, with  gradient-based approaches as predominant methods. 
When modelling multivariate likelihoods, a usual but undesirable outcome is that the learned model fits only a subset of the observed variables, overlooking the rest.
In this work, we study this problem through the lens of multitask learning (MTL), where similar effects have been broadly studied. 
While MTL solutions do not directly apply in the probabilistic setting---as they cannot handle the likelihood constraints---we show that similar ideas may be leveraged during data preprocessing.  
  First, we show that data standardization often helps under common continuous likelihoods, but it is not enough in the general case, specially under mixed continuous and discrete likelihood models. 
In order for \emph{balance multivariate learning}, we then propose  a novel data preprocessing, \emph{Lipschitz standardization}, which balances the local Lipschitz smoothness across variables. %
Our experiments on real-world datasets show that Lipschitz standardization leads to more accurate multivariate models than the ones learned using existing data preprocessing techniques.
The models and datasets employed in the experiments can be found in \href{https://github.com/adrianjav/lipschitz-standardization}{https://github.com/adrianjav/lipschitz-standardization}.

\end{abstract}

\section{Introduction} \label{sec:intro}

\begin{wrapfigure}[12]{r}{.5\textwidth}
	\vspace*{-4em}
	\centering
	\begin{minipage}[c]{.5\columnwidth}
	    \centering
	    \includegraphics[width=.49\textwidth, keepaspectratio]{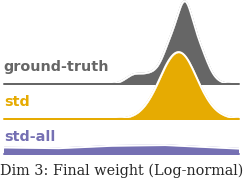} %
	    \includegraphics[width=.49\textwidth, keepaspectratio]{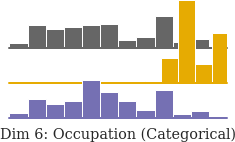}%
	    \caption{Marginals of continuous (left) and discrete (right) variables from the Adult dataset obtained from a trained VAE. Top to bottom: \textcolor{gt}{\textbf{ground-truth}}, actual data; \textcolor{std}{\textbf{std}}, continuous variables were standardized; \textcolor{std-gamma}{\textbf{std-all}}, everything was standardized after replacing the discrete distributions by continuous approximations.\label{fig:intro-1}}
	\end{minipage}%
\end{wrapfigure} %

In the past few years %
gradient-based optimization approaches are becoming the gold standard for probabilistic learning.
Representative examples of this trend include black box variational inference (BBVI)~\citep{ranganath2013black} and  Variational Autoencoders (VAE)~\citep{diederik2014auto}.
However, when such methods are applied to real-world datasets, one often encounters issues such as numerical instabilities. %

As an illustrative example, we learn a VAE on the \textit{Adult dataset} from the UCI repository~\citep{Dua:2019}, where every observation is represented by a set of twelve mixed continuous and discrete variables, with heterogeneous data distributions (see Figure~\ref{fig:intro-1}). 
As it is a common practice, we prevent numerical issues by standardizing the continuous variables prior to training the model. 
However,  as shown in Figure~\ref{fig:intro-1}, while the learned model does a reasonable job at fitting the continuous variable \textit{Final weight},  it results in a poor fit of the discrete variable \textit{Occupation}.
Since discrete data seem  cumbersome to work with, we then rely on a continuous approximation of these variables \textit{and standardize every variable} to learn the VAE. Once the VAE is learned, we use the learned parameters to recover the parameters of the discrete likelihoods. 
In this case, illustrated in the bottom row of Figure~\ref{fig:intro-1},  the VAE does a better job at capturing the \textit{Occupation} but at the price of a poor fitting of the \textit{Final weight}.
In order to understand the source of this issue,  we need to dive deeper into the problem formulation.
In short, the objective function of the VAE can be written as the sum of per-variable losses, i.e., $\elbo = \sum_d \elbo_d$, and thus be interpreted as a  multitask learning (MTL) problem--where different tasks (variables, in our case) compete for the model parameters during learning.  
In this context, previous work has shown that disparities in the gradient magnitudes across tasks, may determine which tasks the model prioritizes during training~\citep{ruder2017overview}. 
Due to the more restrictive nature of probabilistic learning, however, extant solutions from the MTL literature---e.g., GradNorm~\citep{chen2017gradnorm}---do not directly apply, as the likelihood would not integrate to one anymore.

\begin{wrapfigure}[9]{r}{.5\textwidth}
	\vspace*{-3em}
	\centering
	\begin{minipage}[c]{.5\columnwidth}
  \centering
  \includegraphics[width=.49\textwidth, keepaspectratio]{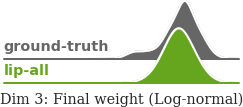} %
  \includegraphics[width=.49\textwidth, keepaspectratio]{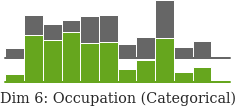}%
  \caption{Same setting as in Figure~\ref{fig:intro-1} where now \textcolor{gt}{\textbf{ground-truth}} is the actual data and \textcolor{ours}{\textbf{lip-all}} refers to the variable fittings obtained after preprocessing with \textit{Lipschitz standardization}, fitting well every variable.\label{fig:intro-2}}
	\end{minipage}%
\end{wrapfigure} %

In this paper, we rely on BBVI as showcase of gradient-based probabilistic learning to show that {the solution resides in the data itself.}
Specifically, in Section~\ref{sec:fair-learning}, we first formalize the concept of \textit{balanced multivariate learning}, which aims to ease that all the observed variables are learned at the same rate,  %
and thus no variable is overlooked. %
In this context, we are able to study why data standardization often helps towards balanced learning when applied to common continuous likelihood functions, such as the Gaussian distribution (Section~\ref{sec:effect}).  
Unfortunately, as shown in our example above, this is not always the case. 
Then, based on our analysis, we propose \textit{Lipschitz standardization} (Section~\ref{sec:lip-std}), a novel preprocessing method that reshapes the data to equalize the  local Lipschitz smoothness of the log-likelihood functions across all continuous and discrete variables. 
As illustrated in Figure~\ref{fig:intro-2}, \textit{Lipschitz standardization} facilitates a more accurate fitting by balancing learning across all variables.  

Finally, we test Lipschitz standardization prior to learning different probabilistic models (mixture models, probabilistic matrix factorization, and VAEs) on six real-world datasets (see Section~\ref{sec:exps}). %
Our results show the effectiveness of the proposed method which leads to a more balanced learning across dimensions, greatly improving the final performance across dimensions on most settings, being in the worst case as good as the best of the considered baseline preprocessing methods, including data standardization.

\section{Problem Statement} \label{sec:ps} 

Let us assume a set of $N$ observations \hbox{$\Xb = \{\xb_n\}_{n=1}^{N}$}, each with $D$ different features \hbox{$\xb_n = \{\x_{nd}\}_{d=1}^D$}.
Following \citet{hoffman2013stochastic}, we consider that 
the joint distribution over the observed variables $\Xb$, local latent variables \hbox{$\Zb = \{ \zb_n \}_{n=1}^{N}$}, and global latent variables  $\betab$, is given by the fairly simple---yet general---latent variable model
    $p(\Xb, \Zb, \betab) = p(\betab) \prod_{n=1}^N p(\xb_n|\zb_n, \betab) p(\zb_n)$.
To account for mixed likelihood models, we further assume that the likelihood factorizes per dimension as %
\begin{equation}\label{eq:lik}
    p(\xb_n|\zb_n, \betab) 
    = \prod_{d=1}^D p_d(\x_{nd}; \etab_{nd}),
\end{equation} %
where  $\etab_{nd}$ denotes the likelihood parameters given by the latent variables $\zb_n$ and $\betab$ for each $\x_{nd}$, $\etab_{nd}(\zb_n, \betab)$. 
Furthermore, we rely on BBVI~\citep{ranganath2013black} 
to approximate the posterior distribution over the latent variables, %
$p(\Zb, \betab | \Xb)$. For simplicity in exposition, we assume a mean-field variational distribution family of the form \hbox{$q(\Zb, \betab) = q_{\gammab_\betab}(\betab) \prod_{n=1}^N q_{\gammab_n}(\zb_n)$}, where {$\{\gammab_n\}_{n=1}^N$} and $\gammab_\betab$ are respectively the local and global variational parameters. We denote by $\gammab$ the set of all variational parameters. 
BBVI relies on (stochastic) gradient ascent to find the parameters that maximize the evidence lower bound (ELBO),\footnote{Or equivalently, that minimize the Kullback-Leibler divergence from $q_{\phi}(\Z, \betab)$ to $p(\Z, \betab | \X)$~\citep{blei2017variational}.}  i.e., 
\begin{equation} 
    \elbo(\Xb, \gammab) = \sum_{d=1}^D \E[q_{\gammab}(\Zb, \betab)]{\log p_d(\xb_d|\Zb,\betab)} -  \KL{q_{\gammab}(\Zb, \betab)}{p(\Zb,\betab)}. \label{eq:elbo}
\end{equation}
BBVI performs iterative updates over the variational (global and local) parameters of the form \hbox{$\gammab^t= \gammab^{t-1} + \alpha {\nabla}_{\gammab}\elbo(\Xb, \gammab, \phib)$} where $t$ is the current step in the optimization procedure.
We further assume that the reparametrization trick~\citep{diederik2014auto} can be applied on the latent variables (i.e., $\Zb, \betab = f(\gammab, \epsilon)$, where $\epsilon$ is a noise variable), such that the gradient of Eq.~\ref{eq:elbo} can be computed as:
\begin{equation} 
	\nabla_\gammab \elbo(\Xb, \gammab) = \sum_{d=1}^D \E[\epsilon]{\nabla_\gammab \ele[d]{\etab_d(\gammab)}} - \nabla_\gammab \KL{q_{\gammab}(\Zb, \betab)}{p(\Zb,\betab)}, \label{eq:gradient}
\end{equation}
where  we denote the log-likelihood $\log p_d(\xb_d; \etab_d(\gammab))$ by $\ele[d]{\etab_d(\gammab)}$, making explicit the dependency of the log-likelihood evaluation to the variational parameters $\gammab$ through the likelihood parameters $\etab$ %
while making implicit its dependency with $\xb_d$ and $\epsilon$.

A closer look to Eq.~\ref{eq:gradient} shows that each dimension in the data contributes to the overall gradient computation in an additive way. 
Therefore, the gradient evaluation with respect to the shared parameters---and in consequence the learning process---can be monopolized by a small subset of dimensions if their gradients dominate this sum in Eq.~\ref{eq:gradient}. 
In other words, while the objective is to capture the joint distribution of \emph{all dimensions}, differences in the gradient evaluation across different observed variables (e.g., Gaussian vs. multinomial) may result in a latent variable model that poorly fits a subset of the observed dimensions, as we already observed in the example of Section~\ref{sec:intro}.%

\subsection{Connections with multitask learning} 
The gradient computation in Eq.~\ref{eq:gradient}---and the undesirable scenario described in the above---may result familiar to those readers knowledgeable about MTL literature. 
In MTL it is common to have a set of shared parameters $\gammab$ whose gradient are of the form $\nabla_\gammab \elbo = \sum_d \nabla_\gammab \elbo_d$, where the sum is taken over all the tasks and each $\elbo_d$ is the loss function of a particular task.
When great disparities exist between task gradients during learning, the resulting model may poorly perform on  some tasks, an effect attributed to the competition between tasks for the shared parameters and known as \emph{negative transfer}~\citep{ruder2017overview}.
Hence, {the (variational) inference problem stated in Eq.~\ref{eq:elbo} may also be interpreted as a (more restrictive) MTL problem} where the input variables play the role of tasks, and the inference parameters are shared.

Given a set of fixed tasks, the most common approach in MTL is to tackle the previous problem using adaptive solutions~\citep{chen2017gradnorm, kendall2018multi, guo2018dynamic}.
These solutions add a set of weights to the loss function, $\elbo = \sum_d \omega_d \elbo_d$, and dynamically change their value---based on different criteria---so that the magnitude of each task gradient $\nabla_\gammab \elbo_d$ is comparable to the ones of other tasks.%

Unfortunately, this type of solutions cannot be applied in the probabilistic setting since, as we mentioned before, we face a more restrictive problem. Specifically, by adding this set of weights in Eq.~\ref{eq:elbo}, we would also modify the likelihood in Eq.~\ref{eq:lik}, which would no longer integrate to one as required.

\subsection{Balanced multivariate learning}\label{sec:fair-learning}

In  variational inference, or more generally, in approximate Bayesian inference, we aim to accurately capture the posterior distribution of the latent variables explaining the joint distribution over {all the observed variables}, and not just a subset of them.
Ideally, we want to follow a \emph{balanced multivariate learning process}, %
where the normalized likelihood improvement per iteration $t+1$ is the same for all dimensions, i.e., 
\begin{equation}\label{eq:ideal}
	\frac{ \ele[d]{\etab_d(\gammab^{t+1})} - \ele[d]{\etab_d(\gammab^{t})}}{\ele[d]{\etab_d(\gammab^{0})}} = C^t,
\end{equation}
for $d=\range{D}$, where $\gammab^{0}$ denotes the initialization of the variational parameters, and $C^t$ the constant improvement at step $t$ for all dimensions. 

This is to the best of our knowledge the first time that balanced learning is properly defined, but its relevance has been acknowledge in prior MTL work (e.g., Eq.~(6) of \citealp{milojkovic2019multi}). 
Of special interest is GradNorm~\citep{chen2017gradnorm}, an adaptive solution whose weights are tuned to ``dynamically adjust gradient norms so different tasks train at similar rates'', including ${\ele[d]{\etab_d(\gammab^{t+1})}}/{\ele[d]{\etab_d(\gammab^{0})}}$ in their formulation.
Unfortunately, Eq.~\ref{eq:ideal} turns out to be an unrealistic goal for the scope of this work.

To find a more feasible objective, we  focus on the class of $L$-smooth functions, which is the broadest class  of functions with convergence guarantees in gradient descent.
A function $\ele{\gammab}$ is $L$-smooth on $Q$ with respect to $\gammab\in Q$ if it is twice-differentiable and, for any $\boldsymbol{a}, \boldsymbol{b} \in Q$, it holds that: 
\begin{equation} \label{eq:lipschitz}
	\norm{\nabla_\gammab \ele{\boldsymbol{a}} - \nabla_\gammab \ele{\boldsymbol{b}}} \leq L~\norm{\boldsymbol{a} - \boldsymbol{b}}.
\end{equation} 
For such class of functions, there exist theoretical results on the convergence rate to a critical point as a function of the Lipschitz constant $L$ and number of steps $T$~\citep{nesterov2018lectures}. 
Using our notation, this rate can be written as
$\min_{t=\range{T}}\norm{\nabla_\gammab \ele[d]{\etab_d(\gammab^t)}} = O(\sqrt{{L}/{T}})$. %
Note that this implies $\norm{\nabla_\gammab \ele[d]{\etab_d(\gammab^t)}}\rightarrow 0$ as $t\rightarrow\infty$, and in turn, \hbox{$\norm{\nabla_\gammab \ele[d]{\etab_d(\gammab^{t+1})} - \nabla_\gammab \ele[d]{\etab_d(\gammab^t)}}\rightarrow 0$}. %
We can thus replace Eq.~\ref{eq:ideal} by
\begin{equation}\label{eq:no-so-ideal}
\frac{\norm{\nabla_\gammab \ele[d]{\etab_d(\gammab_d^{t+1})} - \nabla_{\gammab} \ele[d]{\etab_d(\gammab_d^t)}}}{\norm{\nabla_\gammab \ele[d]{\etab_d(\gammab_d^0)}}} = C^t,%
\end{equation}
which instead focuses on the difference between consecutive gradients to be proportionally equal across dimensions. 
Finally, assuming a good parameter initialization $\gammab^{0}$ such that the initial gradient magnitudes are comparable across dimensions, we can consider constant the denominator from Eq.~\ref{eq:no-so-ideal} as well.
As a result, forcing every dimension to be \hbox{$L$-smooth}%
, i.e., 
\begin{equation}
\norm{\nabla_\gammab \ele[d]{\etab_d(\gammab^{t+1})} - \nabla_{\gammab} \ele[d]{\etab_d(\gammab^t)}} \leq  L~\norm{\gammab^{t+1} - \gammab^{t}} %
\end{equation}
turns out to be a weaker version of Eq.~\ref{eq:no-so-ideal}, whose goal is to ease a more balanced multivariate learning process. 

 In the following section, we study the impact of data standardization on the learning process. To this end, we show the relationship between the Lipschitz constants of the likelihood functions evaluated on the original and the standardized data.   We then propose an estimator of the (local) Lipschitz constant, which 
 allows us to show that, while data standardization may help,  unfortunately in some cases is may counterproductive for balanced multivariate learning.

\section{The effect of standardization} \label{sec:effect}

Preprocessing methods (e.g., standardization) are widely used in statistics and machine learning.
However, there is a priori no way of deciding which one to use \citep{gnanadesikan1995weighting, juszczak2002feature, milligan1988study}. 
In distance-based machine learning methods, e.g. clustering, the effectiveness of these two methods can be readily understood since they bring all the data into a similar range, making the distance between points comparable across dimensions~\citep{aksoy2001feature}. 
In other approaches, such as maximum likelihood or variational inference, the distance argument becomes less convincing,\footnote{In the Bayesian framework, one may also argue that standardization eases the prior selection process (even for those random variables indirectly related with the data), improving the overall performance of the algorithm.} since explicit distance between observations is no longer evaluated. 
Another argument is that they usually improve numerical stability by moving the data, and thus the model parameters, to a well-behaved part of the real space. 
Since computers struggle to work with tiny and large values, this would have an inherent effect in the learning process.

In this section, we study the impact that dimension-wise data preprocessing, specifically scaling transformations of the form $\xtilde = \omega \x$, has on BBVI as an example of Bayesian inference methods based on first order optimization. %
We choose scaling transformations since: i) they preserve important properties of the data distribution, such as domain and tails; and ii) they are broadly used in practice~\citep{han2011data}.
Note that as shifting the data, $\xtilde = \x - \mu$, may violate distributional restrictions (e.g., non-negativity), we  assume that the data may have been already shifted prior to the likelihood selection.
Specifically, our main focus is on three broadly-used data scaling methods:
\begin{align}
\textbf{Standardization: } \xtilde_{nd} = x_{nd}/ \std_d, \quad
\textbf{Normalization: } \xtilde_{nd} = x_{nd} / {\max}_d, \quad
\textbf{Interquartile range: } \xtilde_{nd} = x_{nd} / {\operatorname{iqr}}_d, \nonumber
\end{align}
where $\std_d$, ${\max}_d$, and $\operatorname{iqr}_d$ denote the empirical standard deviation, absolute maximum, and interquartile range of the $d$-th dimension, respectively.

Next, we introduce a novel perspective on the effect of data scaling in inference methods based on first-order optimization.
In a similar way as~\citet{santurkar2018does} showed that batch normalization~\citep{ioffe2015batch} smooths out the  optimization landscape of the loss function, we show that data standardization often 
 smooths out the log-likelihood optimization landscape in a similar way across dimensions.
Importantly, by  applying the chain rule to the gradient computation, i.e., \hbox{$\nabla_\gammab \ele{\etab(\gammab)} = \nabla_\etab \ele{\etab} \cdot \nabla_\gammab \etab$}, we can focus on the data-dependent part, the likelihood gradient $\nabla_\etab \ele{\etab}$.\footnote{We assume the model-dependent part $\nabla_\gammab \etab$ to be similar across dimensions.} 
In the following, we denote by \hbox{$\eletilde[d]{\etatildeb_d} := \log p_d(\xtildeb_d; \etatildeb_d)$}  the likelihood function (with parameters $\etatildeb_d$) evaluated on the scaled data.%

\subsection{Scaling the exponential family}
\label{sec:EF}
Henceforth, we consider each dimension of the observed data to be modeled by a member of the exponential family, i.e., 
\begin{equation}
p_d(\x_{nd} ; \etab_{nd}) = h(\x_{nd}) \exp{\left[\etab_{nd}^\top T(x_{nd}) - A(\etab_{nd})\right]},
\end{equation}
where $\etab_{nd}(\zb_n, \betab)$ denotes the natural parameters parameretized by the latent variables, $T(\x)$ the sufficient statistics, $h(\x)$ is the base measure, and $A(\etab)$ the log-partition function. 
Note that both $\etab$ and $T(x)$ are vectors of size $I_d$. %
Working with the exponential family let us draw one really useful relation between scaled and original data:

\begin{proposition}[Simplified] \label{prop:grad-exp}
	Let $p(\x; \etab)$ be a member of the exponential family where $\x\in\R$ and $\etab\in\R^I$. Besides, let us define $\xtilde := \omega\x$ for a given  $\omega\in\R$. Then, if every sufficient statistic can be factorized as {$T_i(\xtilde) = f_i(\omega)T_i(\x) + g_i(\omega)$}, the following holds: 
	\begin{equation}
	\spd[j]{\etatilde_i} \log p(\xtilde, \etatildeb) = f_i(\omega)^j~\spd[j]{\eta_i} \log p(\x; \etab),
	\end{equation} %
	where $\spd[j]{\eta_i}$ and $\spd[j]{\etatilde_i}$ denote the $j$-th partial derivative with respect to $\eta_i$ and $\etatilde_i := \eta_i / f_i(\omega)$, respectively.
\end{proposition}

A more complex version of the proposition and its proof can be found Appendix~\ref{appendix:A}.
Although the proposition's requirements may look restrictive at first, as reported in Table \ref{table:dists}, many commonly-used distributions fulfil such properties.  
It also is worth-mentioning that in the case of the log-normal distribution we consider the scaling function $\xtilde = \x^\omega$, instead of $\xtilde = \omega\x$. 

 \begin{table*}[t]
 	\centering
     \caption{First two columns: Multiplicative and additive noise (see Prop.~\ref{prop:grad-exp}) for some common distributions (parameterized for simplicity with the canonical parameters, instead of the natural ones). When $f_i$ or $g_i$ is omitted, it is assumed to be $1$ or $0$, respectively. Last two columns: $L$-smoothness of the scaled likelihood (parameterized by $\etatilde_1$ and $\etatilde_2$) as a function of the original (canonical)  likelihood parameters. $Rat$ denotes a rational function, and $\psi^{(1)}$ the trigamma function.} \label{table:dists}
 	\begin{tabular}{lllll} \toprule
 		\textbf{Distribution (param.)}   & \multicolumn{1}{l}{$T_1(x)$}  &  \multicolumn{1}{l}{$T_2(x)$} &  \multicolumn{1}{|l}{$\Ltilde_1^\std$} &  \multicolumn{1}{l}{$\Ltilde_2^\std$} \\ \midrule
 		(Log-)Normal \hfill $(\mu, \sigma)$ & $f_1 = \omega$ & $f_2 = \omega^2$ & $1 + 2\abs{\frac{\mu}{\sigma}}$ & $\approx 4\abs{\frac{\mu}{\sigma}}^2 + 2$ \\ %
 		Gamma \hfill $(\alpha, \beta)$ & $g_1 = \log\omega$ & $f_2 = \omega$ & $\approx\abs{\alpha\psi^{(1)}(\alpha)}$ & $1+1/\sqrt{\alpha}$\\ 
 		Inverse Gaussian \hfill $(\mu, \lambda)$ & $f_2 = 1/\omega$ & $1 + 1/\mu^2$ & $Rat(\mu)$ \\ 
 		Inverse Gamma \hfill $(\alpha, \beta)$ & $g_1 = \log\omega$ & $f_2 = 1/\omega$ & $\approx\abs{\alpha\psi^{(1)}(\alpha)}$ & $Rat(\alpha)$ \\ 
 		Exponential \hfill $(\lambda)$ & $f_1 = \omega$ & & \multicolumn{1}{l}{1} & \\ 
 		Rayleigh \hfill $(\sigma)$ & $f_1 = \omega^2$ & & $\approx 5.428$ & \\ \bottomrule
 	\end{tabular}
 \end{table*}

Assume now that $\ele{\etab}$ is $L_i$-smooth with respect to its $i$-th natural parameter, $\eta_i$.
Using Proposition~\ref{prop:grad-exp}, we  obtain the Lipschitz constant of the scaled likelihood  $\eletilde[d]{\etatildeb_d}$ as a function of the original one $\ele{\etab}$, i.e.,  %
\begin{align*}\label{eq:lip-logp}
    | \partial_{\etatilde_i} \eletilde{\widetilde{\boldsymbol{a}}} - \partial_{\etatilde_i} \eletilde{\widetilde{\boldsymbol{b}}} | &= | f_i(\omega)|~| \partial_{\eta_i} \ele{\boldsymbol{a}} - \partial_{\eta_i} \ele{\boldsymbol{b}} | \leq  |f_i(\omega)|L_i ~\norm{\boldsymbol{a} - \boldsymbol{b}} \\
    &= \abs{f_i(\omega)}L_i ~||{f(\omega) \odot (\widetilde{\boldsymbol{a}} - \widetilde{\boldsymbol{b}})}|| \leq \abs{f_i(\omega)}~\norm{f(\omega)} L_i~||\widetilde{\boldsymbol{a}} - \widetilde{\boldsymbol{b}} || \numberthis,
  \end{align*}
where $\widetilde{\boldsymbol{a}}, \widetilde{\boldsymbol{b}}\in \R^I$ are two different (scaled) parameters and the last expression is a result of the Cauchy-Schwarz inequality.
Assuming the $1$-norm, this implies that the scaled log-likelihood $\eletilde{\etatilde}$ is \hbox{$\Ltilde_i$-smooth} with respect to $\etatilde_i$, with \begin{equation} \label{eq:ltilde}
	\Ltilde_i(\omega)=\abs{f_i(\omega)}~\sum_j\abs{f_j(\omega)}  L_i.
\end{equation}

\subsection{``Standardizing'' the  optimization landscape}

In order to quantify the $L$-smoothness of a function, we need to compute its Lipschitz constant. 
As we are considering here data scaling transformations, i.e.,  a preprocessing step, we focus
on the local $L$-smoothness around 
 the empirical estimation of the natural parameters, denoted by $\widehat{\etab}$. %
 As an example, assuming a Gaussian variable with empirical  mean and standard deviation denoted by $\widehat{\mu}$ and $\widehat{\sigma}$, then  $\widehat{\eta}_1 = \widehat{\mu}/\widehat{\sigma}^2$ and $\widehat{\eta}_2 = -1/2\widehat{\sigma}^2$.

Unfortunately, calculating the ($\epsilon$-local) Lipschitz constant may be  challenging, as it involves solving 
\begin{equation}
    L_i = \max_{\substack{\boldsymbol{a}\neq \boldsymbol{b} \\\boldsymbol{a}, \boldsymbol{b}   \in B(\widehat{\etab}, \epsilon)}} \frac{\norm{\spd{\eta_i} \ele{\boldsymbol{a}} - \spd{\eta_i} \ele{\boldsymbol{b}}}}{\norm{\boldsymbol{a} - \boldsymbol{b}}},
\end{equation}
where  $B(\widehat{\etab}, \epsilon)$ is the ball with radius $\epsilon$ and centered in the empirical estimation of the natural parameters $\widehat{\etab}$.
Instead, we here rely on an  estimator of $L_i$, which 
is derived by taking the limit $\epsilon\rightarrow 0$ and
making
use of the multivariate mean value theorem.%
\begin{theorem}[Mean Value Theorem]
	Let $\ell(\etab)$ be a twice-differentiable real-valued function with respect to $\eta_i \in \etab$ on $Q\subset\R^I$. Then, for any two values \hbox{$\boldsymbol{a}, \boldsymbol{b} \in Q$}, there exists \hbox{$\boldsymbol{c}\in Q$} such that
	$$\spd{\eta_i} \ell(\boldsymbol{a}) - \spd{\eta_i} \ell(\boldsymbol{b}) = \nabla_\etab \left[\spd{\eta_i} \ell(\boldsymbol{c})\right] \cdot (\boldsymbol{a} - \boldsymbol{b}).$$
\end{theorem}
By taking norms above and applying the Cauchy-Schwarz inequality we obtain the same inequality as in Eq.~\ref{eq:lipschitz}, $\norm{\spd{\eta_i} \ell(\boldsymbol{a}) - \spd{\eta_i} \ell(\boldsymbol{b})} \leq \norm{\nabla_\etab \spd{\eta_i} \ell(\boldsymbol{c})} \cdot \norm{\boldsymbol{a} - \boldsymbol{b}}$.
Setting $c = \widehat{\etab}$, we obtain our local estimator of the Lipschitz constant as: 
\begin{equation}\label{eq:estimator}
L_i \approx \norm[1]{\nabla_\etab~\spd{\eta_i} \ele{\widehat{\etab}}} = \sum_j~\abs{\spd{\eta_j\eta_i} \ele{\widehat{\etab}}}.
\end{equation}
Importantly,  if $\ell$ is $L_i$-smooth for each $\eta_i$ in the set of natural parameters $\etab$, then it is $\sum_i L_i$-smooth with respect to $\etab$.
Similarly, if $\ell_1$ is $L_1$-smooth and $\ell_2$ $L_2$-smooth, then $\ell_1 + \ell_2$ is $(L_1 + L_2)$-smooth.\footnote{Note that it could still exist an $L < L_1 + L_2$ such that $\ell_1 + \ell_2$ is $L$-smooth.} These properties are proved in Appendix~\ref{appendix:G}.

Moreover, for the distributions considered in Table~\ref{table:dists}, we can use our estimator to approximate the resulting $L$-smoothness after standardizing the data (details in Appendix~\ref{appendix:C}). 
These results shed some light on why standardizing works well in many settings, since \emph{it makes the $L$-smoothness comparable across dimensions} for several common likelihood functions. 
Specifically, i) the exponential and Rayleigh distributions have constant (local) $L$-smoothness; ii) a centered (log-)normal distribution is $3$-smooth; and iii) the Gamma distribution is (approximately) $1$-smooth as long as its shape parameter  $\alpha$ (which is scale-invariant, i.e., $\tilde{\alpha}=\alpha$) is sufficiently large. %
However, Table~\ref{table:dists} also showcases that for other likelihood the resulting Lipschitz constants may not be comparable.
This is the case for the inverse Gaussian (Gamma) distribution, whose Lipschitz constants after standardizing are rational functions of $\mu$ (of $\alpha$) that can be arbitrarily large or small.

\section{Lipschitz standardization} \label{sec:lip-std}

In the previous section we observed that the Lipschitz constant after scaling the data, $\Ltilde_i(\omega)$, can be seen as a function of the scaling factor $\omega$.    %
As a consequence, it should be possible to find an $\omega$ that eases balanced multivariate learning by making all the dimensions in the data share the same Lipschitz constant. 
In this section, we propose a novel data scaling algorithm with this same goal in mind,  \emph{Lipschitz standardization}. 
Intuitively, our algorithm puts the data into a region of the parameter space where the local $L$-smoothness is comparable across all dimensions.
Given a single $L$-smooth function $\ele{\gammab}$, it can be shown that there exists an optimal step size $\alpha^* = 1/L$ for first-order optimization~\citep{nesterov2018lectures}.
However, when we aim to jointly fit multiple functions, in our case log-likelihood functions $\{\ele[d]{\etab_d(\gammab)}\}_{d=1}^D$, each one being $L_d$-smooth, the optimal learning rate for each individual likelihood is different, {although the parameters (in our case, the variational parameters $\gammab$) that we optimize are shared}.
Importantly, while there exists an optimal learning rate for the overall likelihood function $\ele{\gammab} = \sum_d \ele[d]{\etab_d(\gammab)}$, it may still lead to an unbalanced learning process, and thus, to inaccurate fitting of the data.%

The proposed \emph{Lipschitz standardization} scales each $d$-th dimension using the weight $\omega_d^*$, obtained such that all dimensions share a similar  Lipschitz, i.e.,  
\begin{equation} \label{eq:system-equations}
    \omega_d^* = \argmin_{\omega_d} \left( \sum_{i=1}^{I_d} \Ltilde_{di}(\omega_d) - L^* \right)^2
\end{equation}
where $\Ltilde_{di}(\omega_d)$ are the scaled Lipschitz constants, as in Eq.~\ref{eq:ltilde}, and $L^*$ the target %
$L$-smoothness.
In our experiments we set $L^*$ to $1/(D\alpha)$, where $\alpha$ is the initial learning rate set by the practitioner.
The motivation behind this choice is approximating the resulting overall likelihood $\tilde{L}$-smoothness to the one optimal for a given learning rate, %
being $\tilde{L}=\sum_d \tilde{L}_d \approx \sum_d 1/(D\alpha) = 1/\alpha$. %

\textbf{Remark 1.}
In our experiments, we use Proposition~\ref{prop:grad-exp} and automatic differentiation to approximate the local $L$-smoothness, as well as closed-form solutions and root-finding methods to find the optimal scaling factors $\omega^*_d$ (details in Appendix~\ref{appendix:B}).
However, we recall that gradient descent may be also used to solve the optimization problem in Eq.~\ref{eq:system-equations}. %
As a result, Lipschitz-standardization is applicable to other log-likelihood functions than the ones discussed above, as well as for different 
problems beyond BBVI.

\textbf{Remark 2.} 
Our algorithm is a preprocessing step, and thus the Lipschitz standardized data $\xtilde$, as well as the scaled likelihood functions \hbox{$\eletilde[d]{\etatildeb_d}$}, are used to learn the model parameters (the variational parameters, in our case). However,
during test and deployment,
we ought come back to the original space of the data.
This can be done, in the case of distributions in the exponential family (see Section~\ref{sec:EF}) by using Prop.~\ref{prop:grad-exp}, which shows how to obtain the parameters of the original likelihood function as $\etab = \boldsymbol{f}(\omega) \odot \etatildeb$. %
Appendix~\ref{appendix:F} briefly sketches this idea, providing examples on how our approach applies to the distributions in Table~\ref{table:dists} and to discrete data, which we discuss next. %

\subsection{Discrete data} \label{sec:mixed-data}

Up to this point, our algorithm only applies to continuous data and  likelihood functions.
However,  real-world data often present mixed continuous and discrete data types, as well as likelihood models. 
Next, we extend the proposed Lipschitz-standardization method to discrete data (represented using the natural numbers), assuming discrete distributions such as  Bernoulli, Poisson and categorical distributions. We refer to this new approach as \emph{Gamma trick}.  %

\textbf{Gamma Trick.} This approach (detailed in Appendix~\ref{appendix:F}) can be summarised in four steps: i) transform the discrete data $\x$ to continuous $\xline$ via additive noise, i.e., $\xline = \x + \epsilon$, for which we assume a Gamma likelihood; ii)
apply Lipschitz standardization to $\xline$ to ease more balanced learning; iii) apply the learning process on the scaled data $\xtilde$ to learn the model parameters $\etatilde$; and iv) estimate the parameters of the original discrete distribution using the learned (un-)scaled continuous distribution. 

\textbf{Recovering the parameters of the discrete likelihood.} 
The Bernoulli and Poisson distributions are characterized by their expected value. Hence, to recover their distributional parameters for testing, it is enough to do mean matching between the original distribution and its (un-scaled) Gamma counterpart.
Note that the mean of the discrete variable $\x$ is given by $\mu = \overline{\mu} - \E{\epsilon}$, where $\overline{\mu}$ is the mean of $\overline{\x}$, i.e., ${\overline{\alpha} }/ {\overline{\beta}}$ under the (un-scaled) Gamma distribution with parameters ${\overline{\alpha} }$ and  ${\overline{\beta}}$.
Therefore, we estimate the mean of the Bernoulli distribution as \hbox{$p = \max(0, \min(1, \mu))$}, and the rate of the Poisson distribution as \hbox{$\lambda = \max(\delta,\mu)$}, where $0 < \delta \ll 1$ to ensure that $\lambda$ is positive.

As the categorical distribution has more than one parameter, a \emph{Bernoulli trick} is applied before applying the {Gamma trick}.
The  Bernoulli trick assumes a one-hot representation of the $K$-dimensional categorical distribution and treat each class as an independent Bernoulli distribution, which as shown above is suitable for the Gamma trick.
To recover the parameter of the categorical distribution $\boldsymbol{\pi} = (\range[\pi]{K})$ we individually recover the mean of each Bernoulli class, $\mu_k$, and make sure that they sum up to one, i.e.,
$\pi_k = {\mu_k}/{\sum_{i=1}^K \mu_i}$ for $k = \range{K}$.
Note that, when applying Lipschitz standardization to the categorical distribution, we account for the fact that it has been divided in $K$ Gamma distributions.
As we want all the observed dimensions to be $L^*$-smooth, we group up the new $K$ Gamma distributions and set their objective $L$-smoothness to $L^*/K$, so that they add up to the same $L$-smoothness, i.e., %
$\sum_k \L_k = L^*$.

\begin{figure*}[t]
	\centering
	\subcaptionbox*{}{\includegraphics[width=\linewidth, keepaspectratio]{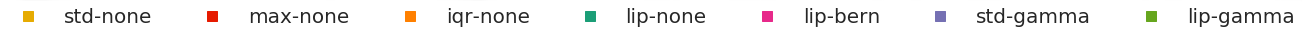}} \\
	\vspace{-\baselineskip}
	\subcaptionbox*{}{\includegraphics[width=.32\linewidth, keepaspectratio]{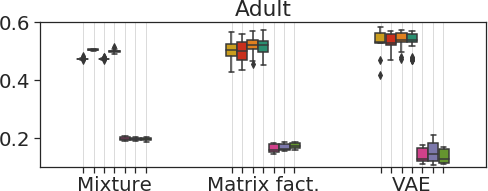}} \hfill %
	\subcaptionbox*{}{\includegraphics[width=.32\linewidth, keepaspectratio]{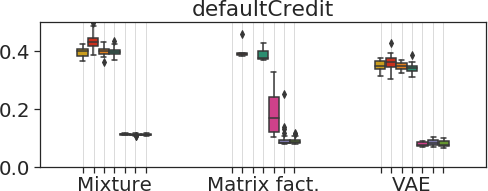}} \hfill %
	\subcaptionbox*{}{\includegraphics[width=.32\linewidth, keepaspectratio]{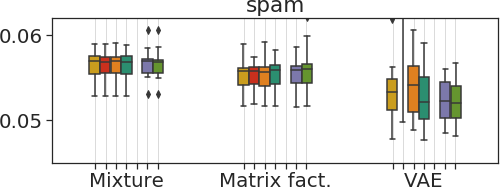}} \\ %
	\vspace{-\baselineskip}
	\subcaptionbox*{}{\includegraphics[width=.32\linewidth, keepaspectratio]{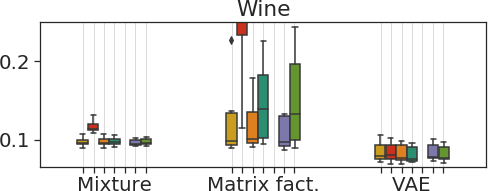}} \hfill %
	\subcaptionbox*{}{\includegraphics[width=.32\linewidth, keepaspectratio]{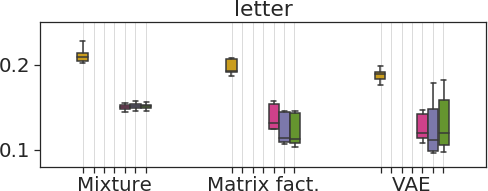}} \hfill %
	\subcaptionbox*{}{\includegraphics[width=.32\linewidth, keepaspectratio]{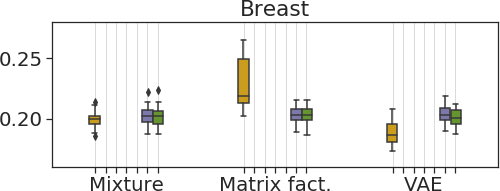}} %
	\vspace{-\baselineskip}
	\caption{Missing imputation error across different datasets and models (lower is better). Each method appears only when applicable and it is shown in the same order as in the legend.} \label{fig:boxplots}
\end{figure*} %

\textbf{Additive noise.} In our transformation from discrete data into continuous data, $\xline = \x + \epsilon$, we ensure that the continuous noise variable $\epsilon$: i) lies in a non-zero measure subset of the unit interval $\epsilon \in (0,1)$ so that the original value is identifiable;  ii) preserves the original data shape as much as possible; and  iii) ensures that the shape parameter $\alpha$ of the Gamma is far from zero, and  $L_1$ does not become arbitrarily large (see Appendix~\ref{appendix:C} for further details).
In our experiments we use noise $\epsilon \sim Beta(1.1, 30)$.

\section{Experiments}

\textbf{Experimental setup} \label{sec:exps}
We use six different datasets from the UCI repository~\citep{Dua:2019} and apply BBVI to %
fit three generative models: i) mixture model; ii) matrix factorization; and iii) (vanilla) VAE.
Additionally, we pick a likelihood for each dimension based on its observable properties (e.g., positive real data or categorical data) and, to provide a fair initialization across all methods and datasets, continuous data is standardized beforehand. 
Appendix~\ref{appendix:D} contains further details and tabular results. 

\textbf{Methods.} 
We consider different combinations of continuous and discrete preprocessing, taking them in our naming nomenclature as prefix and suffix, respectively.
Specifically, for continuous variables we use: i)~\emph{std}, standardization; ii) \emph{max}, normalization; iii)~\emph{iqr}, divides by the interquartile range; iv) \emph{lip}, Lipschitz standardization. %
And similarly we consider for discrete distributions: i)~\emph{none}, leaves the discrete data as it is; ii)~\emph{bern}, applies the Bernoulli trick to categorical data; iii)~\emph{gamma}, applies the Gamma trick to all discrete variables.
As an example, the proposed method applies the Gamma trick to the discrete variables, and then Lipschitz standardizes all the data, so that it is denoted as \lipgamma. %

\begin{wrapfigure}[26]{r}{.5\textwidth} %
    \centering
    \begin{minipage}[c]{.5\textwidth}
    	\subcaptionbox{Mixture model.}{\includegraphics[width=.49\columnwidth, keepaspectratio]{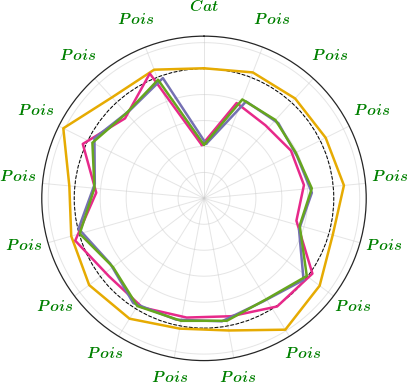}} \subcaptionbox*{}{\includegraphics[width=.3\columnwidth, keepaspectratio]{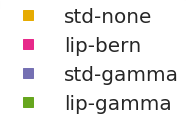}}
    	\subcaptionbox{Matrix factorization.}{\includegraphics[width=.49\columnwidth, keepaspectratio]{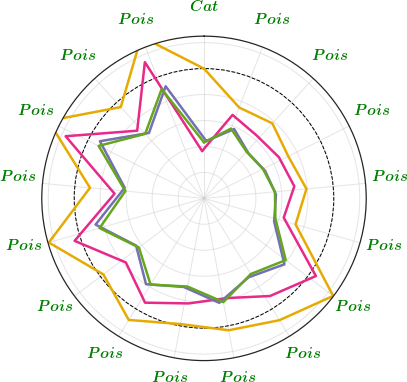}} 
    	\subcaptionbox{VAE.}{\includegraphics[width=.49\columnwidth, keepaspectratio]{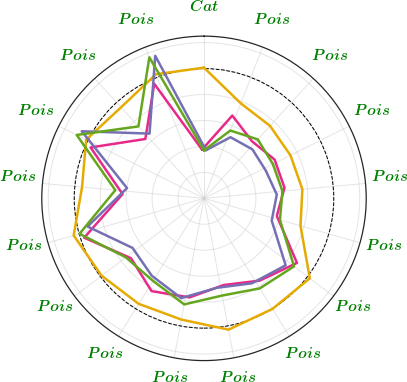}}
    	\caption{Per-dimension normalized error for different models on the \emph{Letter} dataset. Dotted line represents the baseline. Values closer to the origin are better.}
    	\label{fig:stars-discrete}
    \end{minipage}
\end{wrapfigure}

\begin{figure*}[t]%
    \centering
    \includegraphics[width=.85\textwidth, keepaspectratio]{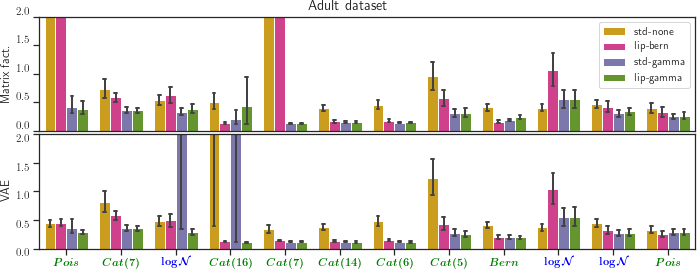}
    \caption{Per-dimension normalized error on the \emph{Adult} dataset. Top row: Matrix factorization. Bottom row: VAE. Note that all methods but \lipgamma overlook a subset of the variables.}
    \label{fig:stars-hetero}
    \vspace{-\baselineskip}
\end{figure*}

\textbf{Metric.}
Analogously to \citet{hivae}, we evaluate the performance of the methods in a data imputation tasks using average missing imputation error as evaluation metric. 
Specifically, normalized mean squared error is used for numerical variables and error rate for nominal ones. %
Besides, in Figures~\ref{fig:stars-discrete} and~\ref{fig:stars-hetero}, we
show the imputation error, normalized by the error obtained by mean imputation, for each dimension.  

\textbf{Results.}
Figure~\ref{fig:boxplots} summarizes the results averaged over three settings with \SI{10}{}, \SI{20}{}, and \SI{50}{\percent} of missing values---with 10 independent runs each---where outliers were removed for better visualization (more detailed results can be found in Appendix~\ref{appendix:E}).
We can distinguish two groups. 
The first group corresponds to the methods that leave discrete data untouched, where we observe that
the Lipschitz standardization (\lipnone) provides comparable results to the best of its counterparts (\maxnone, \stdnone, \iqrnone), being worth-mentioning the results of matrix factorization in \emph{defaultCredit}, where \stdnone and \iqrnone completely disappear from the plot after removing outliers.
Clearly, the second group of methods, which handle discrete variables using either the Bernoulli or Gamma trick, outperform the former group. 
This becomes particularly clear on highly heterogeneous datasets (e.g., \emph{defaultCredit} and \emph{Adult}), where we obtain---and occasionally beat---state-of-the-art results reported by~\citet{hivae}. %

We remark that, while results across models are consistent, the effect of data preprocessing directly depends on the model capacity and dataset complexity.
Specifically, the \emph{mixture model} is too restrictive, finding the same optimum regardless of the preprocessing; \emph{matrix factorization} has enough capacity to be greatly affected by the data (as shown in Figure~\ref{fig:boxplots}); and the VAE is as powerful as to overcome most of the differences in the preprocessing for simpler datasets, yet still being affected for the most complex datasets. %
This is nicely exemplified in Figure~\ref{fig:stars-discrete}, which shows per-dimension normalized error on the \emph{Letter} dataset, where we clearly observe the benefits of both the Bernoulli and Gamma tricks. 

Last but not least, the advantage of using Lipschitz-standardization, i.e. \lipgamma, compared with %
the other two competitive methods, \lipbern and \stdgamma, comes in the form of more
 consistent results for all datasets and independent runs, due to a more balanced learning. %
This can be easily seen by analyzing the per-dimension error of the most complex datasets---see  Figure~\ref{fig:stars-hetero}---where \lipgamma improves the overall imputation error across tasks without completely overlooking any variable. %
On the other hand, both \lipbern and \stdgamma overlook four different variables on the \emph{Adult} dataset using two different models. 
This behavior is not exclusive of \emph{Adult}, as Figures~\ref{fig:barplot-default}-\ref{fig:barplot-letter} and Tables~\ref{tab:10percent}-\ref{tab:50percent} in Appendix~\ref{appendix:D} show.
To tie everything up, we would like to point out that the illustrative example given in Section~\ref{sec:intro} (Figures~\ref{fig:intro-1}-\ref{fig:intro-2}) corresponds to a particular run from the bottom row.

\section{Conclusions}
In this work we have introduced the problem of balanced multivariate learning, which occurs when first-order optimization is used to perform approximate inference in multivariate probabilistic models, and which can be seen as a MTL problem.  %
Then, since existing solutions for MTL problems do not seem to directly apply in the probabilistic setting, we have instead focused on data preprocessing  as a simple and practical solution to mitigate unbalanced learning. 
In particular, we  have shed new insights on the behaviour of data standardization, finding that it makes the smoothness of common continuous  log-likelihoods comparable.
Finally, we have proposed Lipschitz standardization, a data preprocessing algorithm that eases balanced multivariate learning by making the local $L$-smoothness equal across all 
(discrete and continuous) dimensions of the data. %
Our experiments show that  Lipschitz standardization %
outperforms existing methods, and %
specially shines when the data is highly  heterogeneous.

Interesting research avenues include the implementation of Lipschitz standardization in  probabilistic programming pipelines, its use in settings different from BBVI (e.g., HMC), and extending this idea to an online algorithm embedded in the learning process, which takes the model into consideration and  enables the fine-tune of the local Lipschitz during learning.

	\bibliography{references}

\begin{thebibliography}{}

\bibitem[Aksoy and Haralick, 2001]{aksoy2001feature}
Aksoy, S. and Haralick, R.~M. (2001).
\newblock Feature normalization and likelihood-based similarity measures for
  image retrieval.
\newblock {\em Pattern recognition letters}, 22(5):563--582.

\bibitem[Blei et~al., 2017]{blei2017variational}
Blei, D.~M., Kucukelbir, A., and McAuliffe, J.~D. (2017).
\newblock Variational inference: A review for statisticians.
\newblock {\em Journal of the American statistical Association},
  112(518):859--877.

\bibitem[Chen et~al., 2018]{chen2017gradnorm}
Chen, Z., Badrinarayanan, V., Lee, C.-Y., and Rabinovich, A. (2018).
\newblock Gradnorm: Gradient normalization for adaptive loss balancing in deep
  multitask networks.
\newblock In {\em International Conference on Machine Learning}, pages
  794--803. PMLR.

\bibitem[Diederik et~al., 2014]{diederik2014auto}
Diederik, P.~K., Welling, M., et~al. (2014).
\newblock Auto-encoding variational bayes.
\newblock In {\em Proceedings of the International Conference on Learning
  Representations (ICLR)}, volume~1.

\bibitem[Dua and Graff, 2017]{Dua:2019}
Dua, D. and Graff, C. (2017).
\newblock {UCI} machine learning repository.

\bibitem[Gnanadesikan et~al., 1995]{gnanadesikan1995weighting}
Gnanadesikan, R., Kettenring, J.~R., and Tsao, S.~L. (1995).
\newblock Weighting and selection of variables for cluster analysis.
\newblock {\em Journal of Classification}, 12(1):113--136.

\bibitem[Guo et~al., 2018]{guo2018dynamic}
Guo, M., Haque, A., Huang, D.-A., Yeung, S., and Fei-Fei, L. (2018).
\newblock Dynamic task prioritization for multitask learning.
\newblock In {\em Proceedings of the European Conference on Computer Vision
  (ECCV)}, pages 270--287.

\bibitem[Han et~al., 2011]{han2011data}
Han, J., Pei, J., and Kamber, M. (2011).
\newblock {\em Data mining: concepts and techniques}.
\newblock Elsevier.

\bibitem[Hoffman et~al., 2013]{hoffman2013stochastic}
Hoffman, M.~D., Blei, D.~M., Wang, C., and Paisley, J. (2013).
\newblock Stochastic variational inference.
\newblock {\em The Journal of Machine Learning Research}, 14(1):1303--1347.

\bibitem[Ioffe and Szegedy, 2015]{ioffe2015batch}
Ioffe, S. and Szegedy, C. (2015).
\newblock Batch normalization: Accelerating deep network training by reducing
  internal covariate shift.
\newblock {\em arXiv preprint arXiv:1502.03167}.

\bibitem[Jang et~al., 2016]{jang2016categorical}
Jang, E., Gu, S., and Poole, B. (2016).
\newblock Categorical reparameterization with gumbel-softmax.
\newblock {\em arXiv preprint arXiv:1611.01144}.

\bibitem[Juszczak et~al., 2002]{juszczak2002feature}
Juszczak, P., Tax, D., and Duin, R.~P. (2002).
\newblock Feature scaling in support vector data description.
\newblock In {\em Proc. asci}, pages 95--102. Citeseer.

\bibitem[Kendall et~al., 2018]{kendall2018multi}
Kendall, A., Gal, Y., and Cipolla, R. (2018).
\newblock Multi-task learning using uncertainty to weigh losses for scene
  geometry and semantics.
\newblock In {\em Proceedings of the IEEE conference on computer vision and
  pattern recognition}, pages 7482--7491.

\bibitem[Milligan and Cooper, 1988]{milligan1988study}
Milligan, G.~W. and Cooper, M.~C. (1988).
\newblock A study of standardization of variables in cluster analysis.
\newblock {\em Journal of classification}, 5(2):181--204.

\bibitem[Milojkovic et~al., 2019]{milojkovic2019multi}
Milojkovic, N., Antognini, D., Bergamin, G., Faltings, B., and Musat, C.
  (2019).
\newblock Multi-gradient descent for multi-objective recommender systems.
\newblock {\em arXiv preprint arXiv:2001.00846}.

\bibitem[Nazabal et~al., 2018]{hivae}
Nazabal, A., Olmos, P.~M., Ghahramani, Z., and Valera, I. (2018).
\newblock Handling incomplete heterogeneous data using vaes.
\newblock {\em arXiv preprint arXiv:1807.03653}.

\bibitem[Nesterov, 2018]{nesterov2018lectures}
Nesterov, Y. (2018).
\newblock {\em Lectures on convex optimization}, volume 137.
\newblock Springer.

\bibitem[Ranganath et~al., 2014]{ranganath2013black}
Ranganath, R., Gerrish, S., and Blei, D. (2014).
\newblock Black box variational inference.
\newblock {\em Artificial Intelligence and Statistics}, pages 814--822.

\bibitem[Ruder, 2017]{ruder2017overview}
Ruder, S. (2017).
\newblock An overview of multi-task learning in deep neural networks.
\newblock {\em arXiv preprint arXiv:1706.05098}.

\bibitem[Santurkar et~al., 2018]{santurkar2018does}
Santurkar, S., Tsipras, D., Ilyas, A., and Madry, A. (2018).
\newblock How does batch normalization help optimization?
\newblock In {\em Advances in Neural Information Processing Systems}, pages
  2483--2493.

\end{thebibliography}
	
	\appendix\onecolumn

\section{Data workflow and Gamma trick} \label{appendix:F}

It is important to bear in mind the transformation the data follows during the training procedure, as well as what we do with the data at each phase. 
To clarify this in our setting, we provide in Figure~\ref{fig:data-flow} two diagrams describing this procedure for continuous and discrete variables, following the notation of the main paper.
As a summary, data is transformed and scaled, and the scaled natural parameters are learned during training. Whenever evaluation is needed, these parameters are always returned to the space of the original data, that is, $\etatildeb$ is transformed to $\etab$ before evaluating on the space of $\x$.

\begin{figure}[h]
    \centering
    \subcaptionbox{Continuous data.}{    \begin{tikzpicture}[
    	dot/.style={circle,fill,inner sep=0.5pt},
    	ft/.style={circle, draw, minimum width=8mm, inner sep=1pt},
    	filter/.style={circle, draw, minimum width=7mm, inner sep=1pt, path picture={\draw[thick, rounded corners, blue] (path picture bounding box.center)--++(65:2mm)--++(0:1mm);
    			\draw[thick, rounded corners, blue] (path picture bounding box.center)--++(245:2mm)--++(180:1mm);}},
    	mylabel/.style={font=\scriptsize\sffamily},
    	]

    	\node[] (x) {$\xb_d\longrightarrow\xtildeb_d$};
    	\node[below=1ex of x] (eta) {$\etab_d\longleftarrow\etatildeb_d$};

        \draw[orange, ->] (x.east) .. controls (1.1, -0.2) and (1.1, -0.6) .. node[anchor=south, rotate=-90] {\tiny learning} (eta.east);

        \draw[blue, ->] (eta.west) .. controls (-1.1, -0.6) and (-1.1, -0.2)  .. node[anchor=south, rotate=90] {\tiny evaluation} (x.west);
    \end{tikzpicture}} %
    \subcaptionbox{Discrete data.\label{fig:data-flow-discrete}}{    \begin{tikzpicture}[
    	dot/.style={circle,fill,inner sep=0.5pt},
    	ft/.style={circle, draw, minimum width=8mm, inner sep=1pt},
    	filter/.style={circle, draw, minimum width=7mm, inner sep=1pt, path picture={\draw[thick, rounded corners, blue] (path picture bounding box.center)--++(65:2mm)--++(0:1mm);
    			\draw[thick, rounded corners, blue] (path picture bounding box.center)--++(245:2mm)--++(180:1mm);}},
    	mylabel/.style={font=\scriptsize\sffamily},
    	]

    	\node[] (x) at (0,0) {$\xb_d\longrightarrow\overline{\xb}_d\longrightarrow\xtildeb_d$};
    	\node[below=1ex of x] (eta) {$\etab_d\longleftarrow\overline{\etab}_d\longleftarrow\etatildeb_d$};

        \draw[orange, ->] (x.east) .. controls (1.8, -0.2) and (1.8, -0.6) .. node[anchor=south, rotate=-90] {\tiny learning} (eta.east);

        \draw[blue, ->] (eta.west) .. controls (-1.8, -0.6) and (-1.8, -0.2)  .. node[anchor=south, rotate=90] {\tiny evaluation} (x.west);
    \end{tikzpicture}}
    \caption{Schematic working flow used in this work. For training, data is transformed and their natural parameters are learned. To evaluate, the original parameters are recovered from the transformed ones.}
    \label{fig:data-flow}
\end{figure}
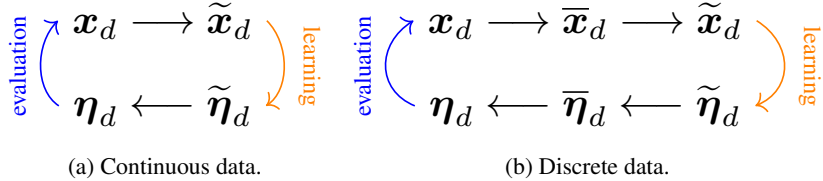

To avoid confusion, let us clarify here what are the transformations described in Figure~\ref{fig:data-flow-discrete} (the continuous case is included as a special case).
The step $\xb_d \mapsto \overline{\xb}_d$ refers to all the transformations regarding discrete data explained in Section~\ref{sec:mixed-data} of the main paper. Specifically, splitting a categorical variable into $K$ independent Bernoulli ones in the case of the Bernoulli trick, and the addition of noise in the case of the Gamma trick.
The transformation $\overline{\xb}_d \mapsto \xtildeb_d$ refers to the data scaling procedure: standardization, normalization, Lipschitz standardization, etc.
The orange arrow is the process performed by the model, which takes the input $\xtildeb_d$ and outputs the parameters~$\etatildeb_d$.
Then, in $\etatildeb_d\mapsto\overline{\etab}_d$, the parameters are scaled back to their original size, using the relationship between natural parameters described in Proposition~\ref{prop:grad-exp} of the main paper.
We do the transformation $\overline{\etab}_d\mapsto\etab_d$ as described in Section~\ref{sec:mixed-data} of the main paper, that is, removing noise, clipping, and gathering the $K$ independent parameters into a dependent one as necessary.
Finally, we can use those parameters $\etab_d$ to evaluate the data coming from the same source as the original data.

Something we have not discussed in the main paper regards the choice of the Gamma distribution as a proxy to learn the parameters of the Bernoulli and Poisson distributions.
As counter-intuitive as it might seem at first, it turns out that the Gamma distribution is a great distribution for doing mean matching with respect to these distributions.
To check this statement, we have run a simple Python code using \texttt{scipy.stats} that: i) generates random samples from a Bernoulli (Poisson) distribution; ii) adds additive noise from a distribution $Beta(1.1, 30)$; iii) fits the data to a Gamma distribution and performs mean matching as explained before; and iv) computes the mean absolute difference between the estimated and real parameters.
This procedure was performed for Bernoulli distributions with parameter $p=i/50$, and Poisson distributions with parameter $\lambda = i$ and $\lambda = i/50$ for $i=0,1,\dots,50$.
The average error obtained was \num{0.0081} and \num{0.0712} for the Bernoulli and Poisson distributions, respectively.

\subsection{Illustrative example of data workflow}

We provide a simple example that shows how data is transformed and used throughout the entire process.
Assume that we have two input dimensions, $D=2$, whose distributions are assumed to be normal $\X_1 \sim \Normal(\mu, \sigma)$ and categorical with 3 classes $\X_2 \sim Cat(\boldsymbol{\pi}  = \left(\pi_1, \pi_2, \pi_3\right))$, respectively. Let us further suppose that we want to use \lipgamma, that is, Lipschitz-standardization combined with the Gamma trick.
Then, we would not alter the first variable $\overline{\X}_1 = \X_1\sim \Normal(\mu, \sigma)$, but substitute $\X_2$ with $\overline{\X}_{2j} = \X_{2j} + \epsilon_j\sim\Gamma(\overline{\alpha}_j, \overline{\beta}_j)$, where $j = 1, 2, 3$ are the indexes of the new variables, $\X_{2j} \sim Bern(p_{j})$ refers to the $j$-th element of $\X_2$ when considered its one-hot representation, and $\epsilon_j \sim Beta(1.1, 30)$ is the (independent) additive noise variable.

Now, we can scale transform all variables, thus obtaining the new scaled variables \hbox{$\widetilde{\X}_1 = \omega_1 \overline{\X}_1 \sim \Normal(\widetilde{\mu}, \widetilde{\sigma})$} and $\widetilde{\X}_{2j} = \omega_{2j} \overline{\X}_{2j} \sim \Gamma(\widetilde{\alpha}, \widetilde{\beta})$ for $j=1, 2, 3$. 
After training---or whenever we need to evaluate the model in non-training data---we ought to return to the original probabilistic model $\X_1, \X_2$.
When recovering the $\overline{\X}$ variables, we need to use Proposition~\ref{prop:grad-exp} so that $\overline{\etab}_i = f_i(\omega) \odot \widetilde{\etab}_i$, where we have obtained $\etatildeb_i$ as the output of our model.

To finally recover the original variables, $\X_1, \X_2$, we do not need to do anything to $\X_1$ since $\X_1 = \overline{\X}_1$. For the second variable, we obtain $\X_{2j} \sim Bern(p_{j})$ as $$p_{j} = \max(0, \min(1, \E{\overline{\X}_{2j}} - \E{\epsilon_{j}})) = \max(0, \min(1, \overline{\alpha}_{j}/\overline{\beta}_{j} - 0.035)),$$
and finally recover $\X_2 \sim Cat(\boldsymbol{\pi})$ with $\boldsymbol{\pi} = \left(\frac{p_1}{p_1 + p_2 + p_3}, \frac{p_2}{p_1 + p_2 + p_3}, \frac{p_3}{p_1 + p_2 + p_3}\right)$.

\section[Basic properties of L-smoothness]{Basic properties of $L$-smoothness} \label{appendix:G}

\begin{proposition}
If a real-valued function $\ele{\etab}$ is $L_i$-smooth with respect to $\eta_i$, the $i$-th parameter of $\etab\in\R^I$, for all $i=\range{I}$, then $\ell$ is $\sum_i L_i$-smooth with respect to $\etab$ (assuming the $1$-norm).
\end{proposition}

\begin{proof}
Consider two arbitrary $\boldsymbol{a}, \boldsymbol{b} \in R^I$. Then, by assumption, $\abs{\spd{\eta_i} \ele{\boldsymbol{a}} - \spd{\eta_i}\ele{\boldsymbol{b}}} \leq L_i \norm{\boldsymbol{a} - \boldsymbol{b}}$ for \hbox{$i=\range{I}$} and%

\begin{equation}
    \norm[1]{\nabla_{\etab} \ele{\boldsymbol{a}} - \nabla_{\etab}\ele{\boldsymbol{b}}} = \sum_i \abs{\spd{\eta_i} \ele{\boldsymbol{a}} - \spd{\eta_i}\ele{\boldsymbol{b}}} \leq \sum_i L_i \norm{\boldsymbol{a} - \boldsymbol{b}}.
\end{equation}
\end{proof}

\begin{proposition}
If two real-valued functions $\ele[1]{\etab}$ and $\ele[2]{\etab}$ are $L_1$-smooth and $L_2$-smooth with respect to $\etab$, respectively, then $\ell_1 + \ell_2$ is $L_1+L_2$-smooth with respect to $\etab$.
\end{proposition}

\begin{proof}
Consider two arbitrary $\boldsymbol{a}, \boldsymbol{b} \in R^I$. Then, 
\begin{align*}
    \norm{\nabla_\etab (\ell_1 + \ell_2)(\boldsymbol{a}) - \nabla_\etab (\ell_1 + \ell_2)(\boldsymbol{b})} &= \norm{(\nabla_\etab \ele[1]{\boldsymbol{a}} - \nabla_\etab \ele[1]{\boldsymbol{b}}) + (\nabla_\etab \ele[2]{\boldsymbol{a}} - \nabla_\etab \ele[2]{\boldsymbol{b}})} \\
    &\leq\norm{\nabla_\etab \ele[1]{\boldsymbol{a}} - \nabla_\etab \ele[1]{\boldsymbol{b}}} + \norm{\nabla_\etab \ele[2]{\boldsymbol{a}} - \nabla_\etab \ele[2]{\boldsymbol{b}}} \\
    &\leq L_1 \norm{\boldsymbol{a} - \boldsymbol{b}} + L_2 \norm{\boldsymbol{a} - \boldsymbol{b}} = (L_1 + L_2) \norm{\boldsymbol{a} - \boldsymbol{b}}
\end{align*}
\end{proof}

\section{Exponential family} \label{appendix:A}

As stated in the main paper, the exponential family is characterized for having the form
\begin{equation}
p_d(\x_{nd} ; \etab_{nd}) = h(\x_{nd}) \exp{\left[\etab_{nd}^\top T(x_{nd}) - A(\etab_{nd})\right]},
\end{equation}
where $\etab_{nd}$ are the natural parameters, $T(\x)$ the sufficient statistics, $h(\x)$ is the base measure, and $A(\etab)$ the log-partition function. 

To ease the task of transforming between natural ($\etab$) and usual ($\boldsymbol{\theta}$) parameters, we provide in Table~\ref{tab:exponential} a cheat-sheet with the relationship between them for the distributions used along the paper, as well as the way that natural parameters are scaled with respect to the scaling factor $\omega$.

\begin{table}
    \centering
    \caption{Relationship between parameters $\boldsymbol{\theta}$ and natural parameters $\etab$, as well as the way the latter scale (see Proposition \ref{prop:grad-exp} of the main text) for different distributions of the exponential family.} \label{tab:exponential}
	\begin{tabular}{cccccccccc}
		\toprule
		\textbf{Likelihood} & $\boldsymbol{\theta}$ & $T(\x)$ &  $\boldsymbol{\theta}\mapsto\etab$ & $\etab\mapsto\boldsymbol{\theta}$ & $\x\mapsto\xtilde$ & $\boldsymbol{f}(\omega)$ & $\etab\mapsto\etatildeb$ \\ \midrule
		Normal & $\begin{bmatrix}\mu\\\\\sigma^2\end{bmatrix}$ & $\begin{bmatrix}\x\\\\\x^2\end{bmatrix}$ &  $\begin{bmatrix}\frac{\mu}{\sigma^2}\\\\\frac{-1}{2\sigma^2}\end{bmatrix}$ & $\begin{bmatrix}\frac{-\eta_1}{2\eta_2}\\\\\frac{-1}{\eta_2}\end{bmatrix}$ & $\omega\x$ & $\begin{bmatrix}\omega\\\\\omega^2\end{bmatrix}$ & $\begin{bmatrix}\frac{\eta_1}{\omega}\\\\\frac{\eta_2}{\omega^2}\end{bmatrix}$ \\ \midrule[0em]
		Log-normal & $\begin{bmatrix}\mu\\\\\sigma^2\end{bmatrix}$ & $\begin{bmatrix}\log\x\\\\(\log\x)^2\end{bmatrix}$ &  $\begin{bmatrix}\frac{\mu}{\sigma^2}\\\\\frac{-1}{2\sigma^2}\end{bmatrix}$ & $\begin{bmatrix}\frac{-\eta_1}{2\eta_2}\\\\\frac{-1}{\eta_2}\end{bmatrix}$  & $\x^\omega$ & $\begin{bmatrix}\omega\\\\\omega^2\end{bmatrix}$ & $\begin{bmatrix}\frac{\eta_1}{\omega}\\\\\frac{\eta_2}{\omega^2}\end{bmatrix}$ \\ \midrule[0em]
		Gamma & $\begin{bmatrix}\alpha\\\\\beta\end{bmatrix}$ & $\begin{bmatrix}\log\x\\\\\x\end{bmatrix}$  &  $\begin{bmatrix}\alpha-1\\\\-\beta\end{bmatrix}$ & $\begin{bmatrix}\eta_1 + 1\\\\-\eta_2\end{bmatrix}$ & $\omega\x$ & $\begin{bmatrix}1\\\\\omega\end{bmatrix}$ & $\begin{bmatrix}\eta_1\\\\\frac{\eta_2}{\omega}\end{bmatrix}$ \\ \midrule[0em]
		Inverse Gaussian & $\begin{bmatrix}\mu\\\\\lambda\end{bmatrix}$ & $\begin{bmatrix}\x\\\\\frac{1}{\x}\end{bmatrix}$ &  $\begin{bmatrix}-\frac{\lambda}{2\mu^2}\\\\-\frac{\lambda}{2}\end{bmatrix}$ & $\begin{bmatrix}\sqrt{\frac{\eta_2}{\eta_1}}\\\\-2\eta_2\end{bmatrix}$ & $\omega\x$ & $\begin{bmatrix}\omega\\\\\frac{1}{\omega}\end{bmatrix}$ & $\begin{bmatrix}\frac{\eta_1}{\omega}\\\\\eta_2\omega\end{bmatrix}$  \\ \midrule[0em]
		Inverse Gamma & $\begin{bmatrix}\alpha\\\\\beta\end{bmatrix}$ & $\begin{bmatrix}\log\x\\\\\frac{1}{\x}\end{bmatrix}$ &  $\begin{bmatrix}-\alpha-1\\\\-\beta\end{bmatrix}$ & $\begin{bmatrix}-\eta_1 - 1\\\\-\eta_2\end{bmatrix}$ & $\omega\x$ & $\begin{bmatrix}1\\\\\frac{1}{\omega}\end{bmatrix}$ & $\begin{bmatrix}\eta_1\\\\\eta_2\omega\end{bmatrix}$ \\ \midrule[0em]
		Exponential & $\begin{bmatrix}\lambda\end{bmatrix}$ & $\begin{bmatrix}\x\end{bmatrix}$  &  $\begin{bmatrix}-\lambda\end{bmatrix}$ & $\begin{bmatrix}-\eta_1\end{bmatrix}$ & $\omega\x$ & $\begin{bmatrix}\omega\end{bmatrix}$ & $\begin{bmatrix}\frac{\eta_1}{\omega}\end{bmatrix}$ \\ \midrule[0em]
		Rayleigh & $\begin{bmatrix}\sigma\end{bmatrix}$ & $\begin{bmatrix}\frac{\x^2}{2}\end{bmatrix}$ &  $\begin{bmatrix}\frac{-1}{\sigma^2}\end{bmatrix}$ & $\begin{bmatrix}\sqrt{\frac{1}{-\eta_1}}\end{bmatrix}$ & $\omega\x$ & $\begin{bmatrix}\omega^2\end{bmatrix}$ & $\begin{bmatrix}\frac{\eta_1}{\omega^2}\end{bmatrix}$  \\ \midrule[0em]
		Bernoulli & $\begin{bmatrix}p\end{bmatrix}$ & $\begin{bmatrix}x\end{bmatrix}$  &  $\begin{bmatrix}\log\frac{p}{1-p}\end{bmatrix}$ & $\begin{bmatrix}\frac{1}{1 + e^{-\eta_1}}\end{bmatrix}$  & \minusmark & \minusmark & \minusmark \\ \midrule[0em]
		Poisson & $\begin{bmatrix}\lambda\end{bmatrix}$ & $\begin{bmatrix}x\end{bmatrix}$  &  $\begin{bmatrix}\log\lambda\end{bmatrix}$ & $\begin{bmatrix}e^{\eta_1}\end{bmatrix}$  &
		\minusmark & \minusmark & \minusmark \\
		\bottomrule
	\end{tabular}
\end{table}

Regarding the relation between scaled and original data in the exponential family, we now prove a more general version of  Proposition~\ref{prop:grad-exp} from the main text.

\begin{proposition} \label{prop:relation}
	Let $p(\x; \etab)$ be a density function of the exponential family where $\x\in X \subset \R$ and $\etab \in Q \subset \R^I$.  Assume a bijective scaling function $\xtilde: X \times \R^+\rightarrow X $ such that for any $\omega\in\R^+$ it defines the function (and random variable) $\xtilde_\omega = \xtilde(x, \omega)$. If all sufficient statistics factorize as $T_i(\xtilde_\omega) = f_i(\omega)T_i(\x) + g_i(\omega)$, then by defining $\etatildeb$ such that $\etab = \boldsymbol{f}(\omega) \odot \etatildeb$, where $\boldsymbol{f} = \left(f_1, f_2, \dots, f_I\right)$ and $\odot$ is the element-wise multiplication, we have
	\begin{equation}
	\spd[j]{\etatilde_i} \log p(\xtilde_\omega, \etatildeb) = f_i(\omega)^j~\spd[j]{\eta_i} \log p(\x; \etab)\quad\text{for}\quad j=1,2,3,\dots,
	\end{equation} %
	where $\spd[j]{\etatilde_i}$ denotes the $j$th-partial derivative with respect to $\etatilde_i$.
\end{proposition}

\begin{proof}
First we are going to relate the normalization constants $A(\etatilde)$ and $A(\eta)$ of $\log p(\xtilde_\omega; \etatilde)$ and $\log p(\x; \eta)$, respectively:
\begin{align*}
    A(\etatilde) &= \log \int h(\xtilde_\omega) \exp \left[ T(\xtilde_\omega)\etatilde \right] \dif \xtilde_\omega = \sum g_i(\omega)\etatilde_i + \log \int h(\xtilde_\omega) \exp \left[ T(\x) \eta \right] \dif \xtilde_\omega   \\
    & = \sum g_i(\omega)\etatilde_i + \log \int \frac{h(\xtilde_\omega)}{h(\x)} h(\x) \exp \left[ T(\x) \eta + A(\eta) - A(\eta) \right] \xtilde'_\omega(\x) \dif \x \\
    & = \sum g_i(\omega)\etatilde_i + A(\eta) + \log \E[p(\x; \eta)]{\frac{h(\xtilde_\omega)}{h(\x)}\xtilde'_\omega(\x)}. \numberthis
\end{align*}

We can safely divide by $h(\x)$ since it is the Radon-Nikodym derivative $\frac{\dif H(\x)}{\dif \x}$ and we can assume that is non-zero almost everywhere in the domain of the likelihood.

Second, we are going to directly relate $p(\xtilde_\omega; \etatilde)$ and $p(\x; \eta)$ using a similar calculation:
\begin{equation}
    p(\xtilde_\omega; \etatilde) = h(\xtilde_\omega) \exp \left[ T(\xtilde_\omega) \etatilde - A(\etatilde) \right] = \frac{h(\xtilde_\omega)}{h(\x)} \frac{h(\x) \exp \left[ T(\x) \eta - A(\eta) \right]}{\E[p(\x; \eta)]{\frac{h(\xtilde_\omega)}{h(\x)} \xtilde_\omega'(x)}} = \frac{h(\xtilde_\omega)}{h(\x)}\frac{p(\x;\eta)}{\E[p(\x; \eta)]{\frac{h(\xtilde_\omega)}{h(\x)}\xtilde_\omega'(\x)}}
\end{equation}

By denoting $\phi(x, \omega)$ everything that is not $p(\x; \eta)$ in the previous equation we have that:
\begin{equation}
    \log p(\xtilde_\omega; \etatilde) = \log p(\x; \eta) + \log \phi (\x, \omega) 
\end{equation}

Now, for the case $j=1$ we just have to use the chain rule and the fact that $\phi(\x, \omega)$ does not depend on $\eta_i$:
\begin{equation}
    \spd{\etatilde_i} \log p(\xtilde_\omega; \etatilde) = \spd{\etatilde_i} \left[ \log p(\x; \eta) + \log \phi(\x, \omega) \right] = \spd{\etatilde_i} \eta_i \spd{\eta_i} \log p(\x; \eta) = f_i(\omega) \spd{\eta_i} \log p(\x; \eta).
\end{equation}

And we can just prove the case $j>1$ by induction:
\begin{equation}
    \spd[j]{\etatilde_i} \log p(\xtilde_\omega; \etatilde) = \spd{\etatilde_i} \spd[j-1]{\etatilde_i} \log p(\xtilde_\omega; \etatilde) = f_i^{j-1}(\omega) \spd{\etatilde_i} \eta_i \spd{\eta_i} \spd[j-1]{\eta_i} \log p(\x; \eta) = f_i^j(\omega) \spd[j]{\eta_i} \log p(\x; \eta).
\end{equation}

\end{proof}

\section{Finding optimal scaling factors for common distributions} \label{appendix:B}

In this section we show some results on how to find the optimal scaling factor $\omega_d$ solving the problem described in Equation~\ref{eq:system-equations} of the main paper. 
For completeness, let us recall the problem:
\begin{equation}
\omega_d^* = \argmin_{\omega_d} \left( \Ltilde_d - L^* \right)^2 = \argmin_{\omega_d} \left( \sum_{i=1}^{I_d} \Ltilde_{di} - L^* \right)^2\quad\text{for}\quad d = \range{D},  \label{eq:optimal-weight}
\end{equation}
where $\Ltilde_d$ is the Lipschitz constant corresponding to the $L$-smoothness of the scaled $d$-th dimension, and $L^* > 0$ is the smoothness goal that we attempt to achieve (as described in the main text).

For common distributions we are able to give some guarantees. Specifically, we can obtain closed-form solutions for the exponential and Gamma distributions, whereas for the (log-)normal distribution we prove the existence and uniqueness of the optimal $\omega_d$.

\paragraph{Remark} We use throughout the proofs the well-known result that $\spd{\eta_i} A(\etab) = \E{T_i(\x)}$ for any $i=\range{I}$ in the case of the exponential family. Therefore, $L_i = \sum_j \spd{\eta_j} \spd{\eta_i} \log p(\x; \etab)$ can be rewritten as \hbox{$L_i = \sum_j \spd{\eta_j} \E[\etab]{T_i(\x)} = \sum_j \spd{\eta_i} \E[\etab]{T_j(\x)}$}, where the last equality is a direct consequence of Young's theorem. %

\begin{proposition}[Exponential distribution]
	Let $X\sim Exp(\lambda)$ and $\Xb = \seq{\x}{n}{N}$. Suppose that, for some value $\widehat{\etab}$, it holds that $\log p(\X; \widehat{\etab})$ is $L_i$-smooth w.r.t. $\eta_i \in \etab$ for $i = 1$. Then the solution for problem \ref{eq:optimal-weight} always exists, is unique, and can be written as
	\begin{equation}
		\omega^* = \sqrt{\frac{L^*}{L_1}}.
	\end{equation}
\end{proposition}

\begin{proof}
	The minimum of problem \ref{eq:optimal-weight} happens when $\sum_{i=1}^{I} \Ltilde_i = L^*$. In this particular case, when $\Ltilde_1 = L^*$.
	As show in Equation~\ref{eq:lip-logp} from the main paper, we know that $\Ltilde_i(\omega) = \abs{f_i(\omega)} \sum_j \abs{f_j(\omega)} L_i$ for the $1$-norm.
	In our particular case, $\Ltilde_1(\omega) = f_1(\omega)^2 L_1 = \omega^2 L_1$.
	
	The resulting equation we need to solve is $L_1\omega^2 = L^*$, whose unique positive solution is $\omega = + \sqrt{\frac{L^*}{L_1}}$.
	
	To show that $\omega^*$ always exists we only have to show that $L_1 > 0$ in all cases, which can easily shown:
	\begin{align*}
		\spd[2]{\eta_1} \log p(\x; \eta_1) &= \spd[2]{\eta_1} \left(\log \lambda - \lambda \x\right) = \spd[2]{\eta_1} \left(\log (-\eta_1) + \eta_1 \x\right) \\
		&= \spd{\eta_1} \left(\frac{1}{\eta_1} + \x\right) = \frac{-1}{\eta_1^2}
	\end{align*}
	and $L_1 = \abs{\spd[2]{\eta_1}} = \eta_1^{-2} > 0$ since $\eta_1 > 0$ by definition.
	
\end{proof}

\begin{proposition}[Gamma distribution]
	Let $X\sim \Gamma(\alpha, \beta)$ and $\Xb = \seq{\x}{n}{N}$. Suppose that, for some value $\widehat{\etab}$, it holds that $\log p(\X; \widehat{\etab})$ is $L_i$-smooth w.r.t. $\eta_i \in \etab$ for $i = 1, 2$. Then the solution for problem \ref{eq:optimal-weight} exists if $L^* > L_1$, is unique, and can be written as
	\begin{equation}
	\omega^* = \frac{-L_1 - L_2 + \sqrt{(L_1 - L_2)^2 + 4L_2L^*}}{2L_2}.
	\end{equation}
\end{proposition}

\begin{proof}
	As in the exponential case, we want to solve the equation $\Ltilde_1(\omega) + \Ltilde_2(\omega) = L^*.$
	\begin{align*}
	\Ltilde_1(\omega) + \Ltilde_2(\omega) &= (\abs{f_1(\omega)} + \abs{f_2(\omega)})(\abs{f_1(\omega)}L_1 + \abs{f_2(\omega)}L_2) = (1 + \omega)(L_1 + L_2\omega) \\
	&= L_2\omega^2 + (L_1 + L_2)\omega + L1 = L^* 
	\end{align*}
	
	Therefore we need to find the roots of the polynomial $L_2\omega^2 + (L_1 + L_2)\omega + L1 - L^* = 0$.
	To find the roots, let us denote the discriminant as $\Delta = (L_1 + L_2)^2 - 4L_2(L_1 - L^*)$.
	Note that we can simplify $\Delta$:
	\begin{align*}
		\Delta &= (L_1 + L_2)^2 - 4L_2(L_1 - L^*) = L_1^2 + L_2^2 + 2L_1L_2 - 4L_1L_2 + 4L_2L^* \\
		&= L_1^2 + L_2^2 - 2L_1L_2 + 4L_2L^* = (L_1 - L_2)^2 + 4L_2L^*.
	\end{align*}
	
	The roots $\omega$ are given by $$\omega = \frac{-L_1 - L_2 \pm \sqrt{\Delta}}{2L_2},$$ and there always exists a single positive root as long as $\sqrt{\Delta} > -L_1 - L_2$:
	\begin{align*}
	\sqrt{\Delta} > -L_1 - L_2 &\Rightarrow \Delta > (L_1 + L_2)^2 \Rightarrow (L_1 - L_2)^2 + 4L_2L^* > (L_1 + L_2)^2 \\
	&\Rightarrow 4L_2L^* > 4L_2L_1 \Rightarrow L^* > L_1.
	\end{align*}
	
	If $L^* > L_1$ we can again show that the solution always exists by computing $L_2$:
	\begin{align*}
		&\spd{\eta_2} A(\etab) = \E[\etab]{T_2(\X)} = \E[\etab]{\X} = \frac{\alpha}{\beta} = -\frac{\eta_1 + 1}{\eta_2} \\
		&\spd[2]{\eta_2} \log p(\x; \etab) = - \spd{\eta_2} \frac{\eta_1 + 1}{\eta_2} = \frac{\eta_1 + 1}{\beta^2} = \frac{\alpha}{\beta^2} > 0 \\
		&\spd{\eta_1} \spd{\eta_2} \log p(\x; \etab) = - \spd{\eta_1} \frac{\eta_1 + 1}{\eta_2} = \frac{1}{-\eta_2} = \frac{1}{\beta} > 0 \\
		&L_2 \approx \abs{\spd{\eta_2} \log p(\x; \etab)} + \abs{\spd{\eta_1} \spd{\eta_2} \log p(\x; \etab)}  > 0
	\end{align*}
	
\end{proof}

\begin{proposition}[Normal distribution] \label{prop:normal}
	Let $X\sim \Normal(\mu, \sigma^2)$ and $\Xb = \seq{\x}{n}{N}$. Suppose that, for some value $\widehat{\etab}$, it holds that $\log p(\X; \widehat{\etab})$ is $L_i$-smooth w.r.t. $\eta_i \in \etab$ for $i = 1, 2$. Then the solution for problem \ref{eq:optimal-weight} always exists, is unique, and can be expressed as the unique positive root of
	\begin{equation}
	Q(\omega) = L_2\omega^4 + (L_1 + L_2)\omega^3 + L_1\omega^2 - L^*.
	\end{equation}
\end{proposition}

\begin{proof}
	First, note that $L_2$ is always positive. To show that we calculate it approximation once again:
	\begin{align*}
		&\spd{\eta_2} A(\etab) = \E[\etab]{T_2(\X)} = \E[\etab]{\X^2} = \mu^2 + \sigma^2 = \frac{\eta_1^2}{4\eta_2^2} + \frac{-1}{2\eta_2} = \frac{\eta_1^2 - 2\eta_2}{4\eta_2^2}\\
		&\spd[2]{\eta_2} \log p(\x; \etab) = \spd{\eta_2} \frac{\eta_1^2 - 2\eta_2}{4\eta_2^2} = \frac{1}{4} \frac{-2\eta_2^2 - 2\eta_2(\eta_1^2 - 2\eta_2)}{\eta_2^4} = \frac{\eta_2-\eta_1^2}{2\eta_2^3} \\
		&\spd{\eta_1} \spd{\eta_2} \log p(\x; \etab) = \spd{\eta_1} \frac{\eta_1^2 - 2\eta_2}{4\eta_2^2} = \frac{\eta_1}{2\eta_2^2} = 2\mu\sigma^2 \\
		&L_2 \approx \abs{\spd[2]{\eta_2} \log p(\x; \etab)} + \abs{\spd{\eta_1} \spd{\eta_2} \log p(\x; \etab)}
	\end{align*}
	We have that $L_2 > 0$ since the second term is only zero when $\mu = 0$ and, if that is the case, $\eta_1 = 0$ and the first term is positive.
	
	As before, we want to solve $\Ltilde_1(\omega) + \Ltilde_2(\omega) = L^*$, which in this case has the form $$(\omega + \omega^2)(L_1\omega + L_2\omega^2) = L_2\omega^4 + (L_1 + L_2)\omega^3 + L_1\omega^2 = L^*.$$
		
	This is equivalent to finding the positive roots of $Q(\omega) = L_2\omega^4 + (L_1 + L_2)\omega^3 + L_1\omega^2 - L^*$.
	Then let us call $P(\omega) = L_2\omega^4 + L_1\omega^2$ so that $Q(\omega) = P(\omega) + (L_1 + L_2)\omega^3 - L^*$.
	
	Note that there exists a unique positive solution of the equation $P(\omega) = G_i$ with $G_i > 0$. In fact, the only positive root of $L_2\omega^4 + L_1\omega^2 - G_i$ is
	\begin{equation} \label{eq:sol-p}
		\omega = + \sqrt{\frac{-L_1 + \sqrt{L_1^2 + 4L_2G_i}}{2L_2}} > 0
	\end{equation}
	
	Define $G_0 = L^*$. As just pointed out, there exists a unique $\omega_1 > 0$ such that $P(\omega_1) = G_0$. Then $$Q(\omega_1) = P(\omega_1) + (L_1 + L_2)\omega_1^3 - G_0 = (L_1 + L_2)\omega_1^3 > 0.$$
	
	Define now $G_1 = G_0 - (L_1 + L_2)\omega_1^3$. Again, there exists a unique $\omega_2 > 0$ such that \hbox{$P(\omega_2) = G_1$ and} $$Q(\omega_2) = P(\omega_2) + (L_1 + L_2)\omega_2^3 - G_0 = G_1 - G_0 + (L_1 + L_2)\omega_2^3 = (L_1 + L_2)(\omega_2^3 - \omega_1^3) < 0$$
	since $G_1 < G_0$, the discriminant of Equation~\ref{eq:sol-p} is smaller in the case of $G_1$ and thus $\omega_2 < \omega_1$.
	
	Define $G_2 = G_1 - (L_1 + L_2)(\omega_2^3 - \omega_1^3)$ and note that $G_1 < G_2 < G_0$ since
	\begin{align*}
		G_2 &= G_1 - (L_1 + L_2)(\omega_2^3 - \omega_1^3) = G_0 - (L_1 + L_2)\omega_1^3 - (L_1 + L_2)(\omega_2^3 - \omega_1^3) \\
		&= G_0 - (L_1 + L_2)\omega_2^3.
	\end{align*}
	
	We can now find $\omega_2 < \omega_3 < \omega_1$ such that $P(\omega_3) = G_2$, $Q(\omega_3) = (L_1 + L_2)(\omega_3^3 - \omega_2^3)$.
	Note that $\omega_1^3 > \omega_3^3 \Rightarrow \omega_1^3 + \omega_2^3 > \omega_3^3 \Rightarrow \omega_1^3 > \omega_3^3 - \omega_2^3$, meaning that $Q(\omega_3) < Q(\omega_1)$.
	
	Thus far, we have built a sequence such that $Q(\omega_2) < 0 < Q(\omega_3) < Q(\omega_1)$. 
	If we follow the process and define $G_3 = G_2 - (L_1 + L_2)(\omega_3^3 - \omega_2^3)$ we will find an $\omega_2 < \omega_4 < \omega_3$ such that $Q(\omega_2) < Q(\omega_4) < 0 < Q(\omega_3) < Q(\omega_1)$.
	
	Finally, let us define the sequence of intervals $I_i = [Q(\omega_{i+1}), Q(\omega_i)]$ for $i = \range{\infty}$ constructed using the described procedure. 
	This sequence is a strictly decreasing nested sequence of non-empty compact subsets of $\mathbb{R}$.
	Therefore, Cantor's intersection theorem states that the intersection of these intervals is non-empty, $\cap_i I_i \neq \emptyset$, and since the only element which is in all the intervals is $0$, $\cap_i I_i = \left\{0\right\}$.
	
	The sequence $\left\{Q(\omega_{2i})\right\}_{i=1}^\infty$ ($\left\{Q(\omega_{2i+1})\right\}_{i=1}^\infty$) converges to $0$ since it is a strictly decreasing (increasing) sequence lower-bounded (upper-bounded) by $0$.
	The sequences of their anti-images, $\left\{\omega_{2i}\right\}_{i=1}^\infty$ and $\left\{\omega_{2i+1}\right\}_{i=1}^\infty$, converge then to the same value, $\omega^*$, the root of $Q$ and the solution of problem \ref{eq:optimal-weight}.

\end{proof}

\section[L-smoothness estimation]{$L$-smoothness estimation} \label{appendix:C}

\subsection[L-smoothness after standardization]{$L$-smoothness after standardization}

Similar to what we have done in Appendix~\ref{appendix:B}, here we are going to compute the estimator of the local $L$-smoothness for some usual distributions using $\L = \sum_i \L_i = \sum_i \sum_j \abs{\spd{\eta_j}\spd{\eta_i} \log p(\x; \etab)}$, and then see how this smoothness changes as we scale by $\omega = 1/\std$. 
We will use here the standard deviation expression of each particular likelihood, therefore these results hold as long as the selected likelihood properly fits the data.

\paragraph{(Log-)Normal distribution}
First, we compute the partial derivatives of the log-likelihood:
\begin{align*}
    &\spd{\eta_1} A(\etab) = \E[\etab]{T_1(\X)} = \E[\etab]{\X} = \mu = \frac{-\eta_1}{2\eta_2} \\
	&\spd{\eta_2} A(\etab) = \E[\etab]{T_2(\X)} = \E[\etab]{\X^2} = \mu^2 + \sigma^2 = \frac{\eta_1^2}{4\eta_2^2} + \frac{-1}{2\eta_2} = \frac{\eta_1^2 - 2\eta_2}{4\eta_2^2}\\
	&\spd[2]{\eta_1} \log p(\x; \etab) = \spd{\eta_1} \frac{-\eta_1}{2\eta_2} = \frac{-1}{2\eta_2} = \sigma^2 \\
	&\spd[2]{\eta_2} \log p(\x; \etab) = \spd{\eta_2} \frac{\eta_1^2 - 2\eta_2}{4\eta_2^2} = \frac{1}{4} \frac{-2\eta_2^2 - 2\eta_2(\eta_1^2 - 2\eta_2)}{\eta_2^4} = \frac{\eta_2-\eta_1^2}{2\eta_2^3} = 2\sigma^2(\sigma^2 + 2\mu^2)\\
	&\spd{\eta_2} \spd{\eta_1} \log p(\x; \etab) = \spd{\eta_1} \spd{\eta_2} \log p(\x; \etab) =  \spd{\eta_1} \frac{\eta_1^2 - 2\eta_2}{4\eta_2^2} = \frac{\eta_1}{2\eta_2^2} = 2\mu\sigma^2 \\
\end{align*}
Therefore, we have that $L_1 \approx \sigma^2 + 2\abs{\mu}\sigma^2 $ and $L_2 \approx 2\sigma^2 ( \abs{\mu} + \sigma^2 + 2\mu^2)$.
After standardizing the data, we have that $\widetilde{\mu} = \mu / \sigma$ and $\widetilde{\sigma}^2 = 1$, resulting in $\Ltilde_1^\std = 1 + 2\frac{\abs{\mu}}{\sigma}$ and $\Ltilde_2^\std = 4\abs{\frac{\mu}{\sigma}}^2 + 2 \frac{\abs{\mu}}{\sigma} + 2$.

\paragraph{Gamma distribution} In this case we have:
\begin{align*}
    \spd{\eta_1} A(\etab) &= \E{T_1(\x)} = \alpha - \log\beta + \log\Gamma(\alpha) +(1-\alpha)\psi(\alpha) \\
    &= \eta_1 + 1 - \log(-\eta_2) + \log\Gamma(\eta_1+1) - \eta_1\psi(\eta_1+1)
\end{align*} %
\begin{align*}
    \spd[2]{\eta_1} \log p(\x; \etab) &= \spd{\eta_1} \left[\eta_1 + 1 - \log(-\eta_2) + \log\Gamma(\eta_1+1) - \eta_1\psi(\eta_1+1) \right] \\
    &=1 + \psi(\eta_1 + 1) - \psi(\eta_1 + 1) - \eta_1\psi^{(1)}(\eta_1 + 1) \\&= 1 - \eta_1\psi^{(1)}(\eta_1 + 1) = 1 + (1 - \alpha) \psi^{(1)}(\alpha) \numberthis \label{eq:gamma-L1}
\end{align*}
\begin{align*}
    & \spd{\eta_2} A(\etab) = \E{T_2(\x)} = \E{\x} = \frac{\alpha}{\beta} = \frac{\eta_1+1}{-\eta_2} \\
	&\spd[2]{\eta_2} \log p(\x; \etab) = \spd{\eta_2} \frac{\eta_1+1}{-\eta_2} = \frac{\eta_1 + 1}{\eta_2^2} = \alpha/\beta^2 = Var\left[\x\right]\\
	&\spd{\eta_2} \spd{\eta_1} \log p(\x; \etab) = \spd{\eta_1} \spd{\eta_2} \log p(\x; \etab) = \frac{1}{-\eta_2} = 1/\beta
\end{align*}

So that $L_1 \approx \abs{1 + (1-\alpha)\psi^{(1)}(\alpha)} + 1/\beta$ and $L_2 \approx Var\left[\x\right] + 1/\beta$.
After standardizing $\widetilde{\alpha} = \alpha$, $\widetilde{\beta} = \sqrt{\alpha}$ and $Var\left[\x\right] = 1$, therefore $\Ltilde_1^\std$ is a function of $\psi^{(1)}(\alpha)$ and $\Ltilde_2^\std = 1 + 1/\sqrt{\alpha}$.

\paragraph{Exponential distribution} If $\X \sim Exp(\lambda)$ then $\X \sim \Gamma(1, 1/\lambda)$, so we can use the previous results so that $L_1 \approx Var\left[\x\right]$ and $\Ltilde_1^\std = 1$.

\paragraph{Rayleigh distribution} This distribution has parameter $\sigma > 0$, sufficient statistic $T_1(\x) = \x^2/2$, and natural parameter $\eta_1 = -1/\sigma^2$.

We start by computing $\spd{\eta_1} A(\etab) = \E{T_1(\x)} = \frac{1}{2} \E{\x^2}$. Using that, for this distribution, \hbox{$\E{\x^j} = \sigma^j 2^{j/2} \Gamma(1 + \frac{j}{2})$}:
\begin{align*}
    & \spd{\eta_1} A(\etab) = \frac{1}{2}\E{\x^2} = \frac{1}{2}\sigma^2 2 \Gamma(2) = \sigma^2 = \frac{-1}{\eta_1} \\
    & \spd[2]{\eta_1} \log p(\x; \etab) = \spd{\eta_1} \frac{-1}{\eta_1} = \frac{1}{\eta_1^2} = \sigma^4
\end{align*}

Therefore, $\L_1 \approx \sigma^4$. After standardization, $Var\left[\x\right] = \frac{4-\pi}{2}\sigma^{2} = 1 \Rightarrow \widetilde{\sigma}^2 = \frac{2}{4-\pi}$ and \hbox{$\Ltilde_1^\std = \left(\frac{2}{4-\pi}\right)^2 \approx 5.428$}.

\paragraph{Inverse Gaussian distribution} This distribution has parameters $\mu, \lambda > 0$, sufficient statistics \hbox{$T_1(\x) = \x$}, $T_2(\x) = 1/\x$, and natural parameters $\eta_1 = \frac{-\lambda}{2\mu^2}, \eta_2 = \frac{-\lambda}{2}$.

\begin{align*}
    & \spd{\eta_1} A(\etab) = \E{\x} = \mu = \sqrt{{\eta_2}/{\eta_1}} \\
    & \spd{\eta_2} A(\etab) = \E{\frac{1}{\x}} = \frac{1}{\mu} + \frac{1}{\lambda} = \sqrt{\frac{\eta_1}{\eta_2}} - \frac{1}{2\eta_2} \\
    & \spd[2]{\eta_1} \log p(\x; \etab) = \spd{\eta_1} \sqrt{\frac{\eta_2}{\eta_1}} = \sqrt{\eta_2} \spd{\eta_1} \frac{1}{\sqrt{\eta_1}} = \frac{-1}{2} \sqrt{\frac{\eta_2}{\eta_1}} \frac{1}{\eta_1} = \sqrt{\frac{\eta_2}{\eta_1}} \frac{\eta_2}{\eta_1} \frac{1}{-2\eta_2} = \mu^3 / \lambda \\
    & \spd{\eta_2}\spd{\eta_1} \log p(\x; \etab) = \spd{\eta_2} \sqrt{\frac{\eta_2}{\eta_1}} = \frac{1}{2} \frac{1}{\sqrt{\eta_1\eta_2}} = \sqrt{\frac{\eta_2}{\eta_1}} \frac{-1}{-2\eta_2} = -\mu / \lambda \\
    & \spd[2]{\eta_2} \log p(\x; \etab) = \spd{\eta_2} \left(\sqrt{\frac{\eta_1}{\eta_2}} - \frac{1}{2\eta_2} \right) = \frac{-1}{2} \sqrt{\frac{\eta_1}{\eta_2}} \frac{1}{\eta_2} + \frac{1}{2\eta_2^2} = \frac{1 - \sqrt{\eta_1\eta_2}}{2\eta_2^2} = \frac{2\mu + \lambda}{\mu\lambda^2}
\end{align*}

Therefore, $L_1\approx \mu^3/\lambda + \mu/\lambda$ and $L_2 \approx \mu/\lambda + (2\mu + \lambda)/(\mu\lambda^2)$.
After standardizing we have that $Var\left[\widetilde{\x}\right] = \mu^3/\lambda = 1 \Rightarrow \lambda = \mu^3$, thus $\Ltilde_1^\std = 1 + 1/\mu^2$ and $\Ltilde_2^\std = (2 + \mu^2 + \mu^4)/\mu^6$.

\paragraph{Inverse Gamma distribution} This distribution has parameters $\alpha, \beta > 0$, sufficient statistics \hbox{$T_1(\x) = \log\x$}, $T_2(\x) = 1/\x$, and natural parameters $\eta_1 = -\alpha - 1, \eta_2 = -\beta$.
\begin{align*}
    \spd{\eta_1} A(\etab) &= \E{T_1(\x)} = \alpha - \log\beta + \log\Gamma(\alpha) -(1+\alpha)\psi(\alpha) \\
    &= -\eta_1 - 1 - \log(-\eta_2) + \log\Gamma(-\eta_1-1) + \eta_1\psi(-\eta_1-1)
\end{align*} %
\begin{align*}
    \spd[2]{\eta_1} \log p(\x; \etab) &= \spd{\eta_1} \left[-\eta_1 - 1 - \log(-\eta_2) + \log\Gamma(-\eta_1-1) + \eta_1\psi(-\eta_1-1) \right] \\
    &= -1 - \psi(-\eta_1 - 1) + \psi(-\eta_1 -1) - \eta_1 \psi^{(1)}(-\eta_1 - 1) \\
    &= -1 + (\alpha + 1) \psi^{(1)}(\alpha)
\end{align*}
\begin{align*}
    & \spd{\eta_2} A(\etab) = \E{T_2(\x)} = \E{1/\x} = \frac{\alpha}{\beta} = \frac{\eta_1+1}{\eta_2} \\
	&\spd[2]{\eta_2} \log p(\x; \etab) = \spd{\eta_2} \frac{\eta_1+1}{\eta_2} = -\frac{\eta_1 + 1}{\eta_2^2} = \frac{\alpha}{\beta^2} \\
	&\spd{\eta_2} \spd{\eta_1} \log p(\x; \etab) = \spd{\eta_1} \spd{\eta_2} \log p(\x; \etab) = \frac{1}{\eta_2} = \frac{1}{-\beta}
\end{align*}
Therefore, $L_1\approx \abs{1 - (\alpha + 1) \psi^{(1)}(\alpha)} + 1/\beta$ and $L_2 \approx 1/\beta + \alpha/\beta^2$. 
After standardizing we obtain 
\begin{align*}
    & Var\left[\widetilde{\x}\right] = \frac{\beta^2}{(\alpha-1)^2(\alpha-2)} = 1 \Rightarrow \beta^2 = (\alpha-1)^2(\alpha-2) \\
    & \Ltilde_2^\std = ((\alpha-1) \sqrt{\alpha-2} + \alpha) / ((\alpha - 1)^2 (\alpha-2)) \\
    & \Ltilde_1^\std = \abs{(\alpha+1)\psi^{(1)}(\alpha) - 1} + 1/((\alpha-1)\sqrt{\alpha-2})
\end{align*}

The interesting bit about these last two estimators is that both explode as they get closer to $2$, and both vanish as they get further from it, as it can be readily checked by plotting them.

\subsection{Scale-invariant smoothness of the Gamma distribution}

In section \ref{sec:mixed-data} it was introduced the concept of Gamma trick, which acts as a approximation for discrete distributions. Moreover, the discrete variables were assumed to take place in the natural numbers. The reason is that it is beneficial for this approximation that the original variable $\x$ is somewhat far from zero.

This statement it is justified by the following: the second derivative of a Gamma log-likelihood with respect to the first natural parameter, $\spd[2]{\eta_1} \log p(x; \eta)$, rapidly decreases as the data moves away from zero.

As computed before in Equation~\ref{eq:gamma-L1}, one part of $L_1$ is scale-invariant and has the form \hbox{$1 + (1 - \alpha) \psi^{(1)}(\alpha)$}.
Figure~\ref{fig:trigamma} shows a plot of this formula as a function of $\alpha$. 
It is easy to observe that as the shape parameter grows the value of (our approximation to) $L_1$ drastically decreases.

\begin{figure}[!hbtp]
    \centering
    \includegraphics[width=.8\textwidth, keepaspectratio]{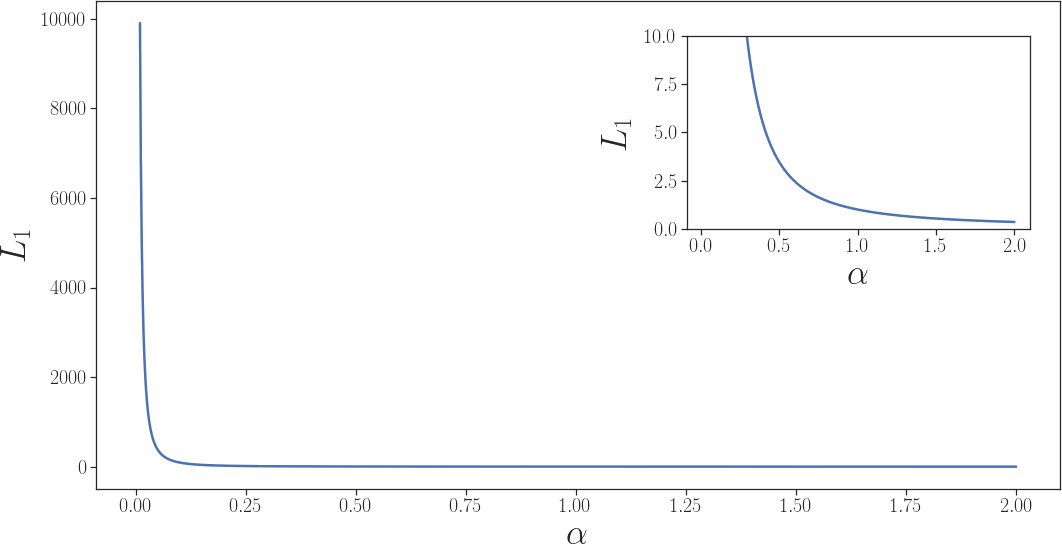}
    \caption{Plot of $L_1$ for the Gamma distribution.}
    \label{fig:trigamma}
\end{figure}

Finally, by supposing that discrete data are natural numbers, the mode is at least one, which in practice means that the value for $\alpha$ is bigger than $1$ (usually close to $10$), thus ensuring that the value of (our approximation to) $L_1$ mostly depends on the scale-dependent parameter $\beta$.

\section{Details on the experimental setup} \label{appendix:D}

\subsection{Missing imputation models}

Here we give a deeper description of the models used on the experiments. All of them have the form described in the problem statement (Section~\ref{sec:ps}), following the graphical model depicted in Figure~\ref{fig:graphical-model}.

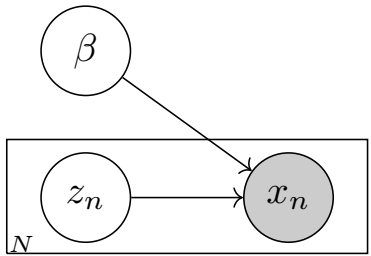
\begin{figure}[h]
    \centering
        \begin{tikzpicture}[
    	dot/.style={circle,fill,inner sep=0.5pt},
    	ft/.style={circle, draw, minimum width=8mm, inner sep=1pt},
    	filter/.style={circle, draw, minimum width=7mm, inner sep=1pt, path picture={\draw[thick, rounded corners, blue] (path picture bounding box.center)--++(65:2mm)--++(0:1mm);
    			\draw[thick, rounded corners, blue] (path picture bounding box.center)--++(245:2mm)--++(180:1mm);}},
    	mylabel/.style={font=\scriptsize\sffamily},
    	]
    	
    	\node[ft, fill=black!20] (x) {$\x_n$};
    	\node[ft, left=10mm of x](z) {$\z_n$};
    	\draw[->] (z.east) -- (x.west);
    	
    	\draw ([yshift=-5mm,xshift=-3mm]z.west) -- ([yshift=-5mm,xshift=3mm]x.east) -| ([yshift=5.2mm,xshift=3mm]x.east) -- ([yshift=5.2mm,xshift=-3mm]z.west) -| ([yshift=-5mm,xshift=-3mm]z.west);
    	
    	\node[below left=-0.7mm and 0.5mm of z] (dummy) {\tiny $N$};
    	
    	\node[ft, above=5mm of z] (beta) {$\beta$};
    	\draw[->] (beta) -- (x);

    \end{tikzpicture}
    \caption{Latent variable model describing the joint distribution of Section~\ref{sec:ps}.}
    \label{fig:graphical-model}
\end{figure}

\paragraph{Mixture model}
Following the form of the join distribution from Section~\ref{sec:ps}, the mixture model is fully described by:
\begin{itemize}
    \item Priors: $$p(\pi_n) = \mathcal{U}(K) \qquad p(\beta) = \Normal(0_K, I_K)$$
    \item Posteriors: $$q_\phi(\z_n) = Cat(\pi_n) \qquad q_{\phi}(\beta) = \Normal(\mu, \Sigma)$$
    \item Linking function: $$\eta(\z_n, \beta_d) = \z_n \beta_d$$
\end{itemize}

Where $\pi_n$ are $K$-dimensional vectors and $z_n$ are one-hot encoding vectors of size $K$.

To ensure that the parameters fulfil the domain restriction of each particular distribution, the following transformations are performed after the linking function is applied:

\begin{itemize}
    \item Greater than $l$: $$ {\eta}' = \operatorname{softplus}(\eta) + l + \num{1e-15}$$
    \item Smaller than $u$: $$ {\eta}' = -(\operatorname{softplus}(\eta) + u + \num{1e-15})$$
\end{itemize}

When it comes to experiments the only hyper-parameter for this model is the number of clusters, $K$. In particular, we use $K=5$ if the dataset is \textit{Breast}, \textit{Wine}, or \textit{spam}, and $K=10$ otherwise.

In order to implement the discrete latent parameters such that they can be trained via automatic differentiation, the latent categorical distribution is implemented using a GumbelSoftmax distribution~\citep{jang2016categorical} with a temperature that updates every $20$ epochs as:
\begin{equation*}
    temp = \max(0.001, e^{-0.001epoch})
\end{equation*}

\paragraph{Matrix factorization}
Similar to the mixture model, the matrix factorization model follows the same graphical model and it is (almost) fully described by:
\begin{itemize}
    \item Priors: $$p(\mu_n) = \Normal(0_K, I_K) \qquad p(\beta) = \Normal(0_K, I_K)$$
    \item Posteriors: $$q_\phi(\z_n) = \Normal(\mu_n, \sigma) \qquad q_{\phi}(\beta) = \Normal(\mu, \Sigma)$$
    \item Linking function: $$\eta(\z_n, \beta_d) = \z_n \beta_d$$
\end{itemize}

There some details that have to be noted. 
First, the variance of the local parameters, $\sigma$, is shared among instances and learnt as a deterministic parameter.
In the same way, only the first parameter, $\eta_1$, of each distribution is learnt following this scheme. The remaining parameters are learnt using gradient descent as deterministic parameters. %

The same transformations as in the mixture model are performed to the parameters in order to fulfil their particular domain requirements.

When it comes to experiments, the only hyper-parameter is the latent size, $K$. In particular, we set it automatically as half the number of dimensions of each dataset (before applying any trick to the data that may increase the number of dimensions).

\paragraph{Variational Auto-Encoder}
We follow the structure of a vanilla VAE~\citep{diederik2014auto} with the following components:
\begin{itemize}
    \item Encoder: 3-layer neural network with hyperbolic tangents as activation functions.
    \item Decoder: 4-layer neural network with ReLU as activation functions.
\end{itemize}

General notes:
\begin{itemize}
    \item We assume normal latent variables with a standard normal as prior.
    \item Hidden layers have 256 neurons.
    \item The latent size is set to the \SI{75}{\percent} of the data number of dimensions (before preprocessing).
    \item Layers are initialized using a Xavier uniform policy.
\end{itemize}

Specifics about the encoder:
\begin{itemize}
    \item As we have to avoid using the missing data (since it is going to be our test set), we implement an input-dropout layer as in \citet{hivae}.
    \item In order to guarantee a common input (and thus, a common well-behaved neural net) across all data scaling methods, we put a batch-normalization layer at the beginning of the encoder. Note that this does not interfere with the goal of this work, which is about the evaluation of the loss function.
    \item In order to obtain the distributional parameters of $\z_n$, $\mu_n$ and $\sigma_n$, we pass the result of the encoder through two linear layers, one for the mean and another for the log-scale. The latter is transformed to the scale via a softplus function.
\end{itemize}

Specifics about the decoder:
\begin{itemize}
    \item The decoder output size is set to the number of parameters to learn. Each one being transformed accordingly with softplus functions to fulfil their distributional restrictions, as done for the other models.
\end{itemize}

\subsection{Experimental setup}

For the experiments we train with Adam and a learning rate of \num{1e-3} for all models but matrix factorization, which is set to \num{1e-2}. Batch size is set to \num{1024} in all cases. We train for \num{400} epochs for the biggest datasets (\textit{letter}, \textit{Adult}, and \textit{defaultCredit}), \num{2000} epochs for the intermediate ones (\textit{Wine}, and \textit{spam}), and \num{3000} epochs for the smallest one (\textit{Breast}). Table \ref{tab:types} describes the types of data across datasets as well as their sizes.

\begin{table}[!hbtp]
	\centering
	\caption{Types of random variables per dimensions and number of samples.} \label{tab:types}
	\begin{tabular}{ccccccc} \toprule
		\textbf{Dataset} & Credit & Adult & Wine & spam & Letter & Breast \\ \midrule
		Continuous & \num{13} & \num{3} & \num{11} & \num{57} & \num{0} & \num{0} \\
		Poisson & \num{1} & \num{2} & \num{1} & \num{0} & \num{16} & \num{9} \\
		Categorical & \num{10} & \num{7} & \num{1} & \num{1} & \num{1} & \num{1} \\ \midrule
		\textbf{No. samples} & \num{30000} & \num{32000} & \num{7000} & \num{4600} & \num{20000} & \num{700} \\ \bottomrule
	\end{tabular}
\end{table}

We automate the process of choosing a likelihood based on basic properties of the data:
\begin{align*}
    &\textbf{Real-valued:} & \x_d\sim \mathcal{N}(\mu, \sigma) \\
    &\textbf{Positive real-valued:} &\x_d\sim \log \mathcal{N}(\mu, \sigma) \\
    &\textbf{Count:} &\x_d\sim \operatorname{Poiss}(\lambda) \\ 
    &\textbf{Binary:} &\x_d\sim \operatorname{Bern}(p) \\
    &\textbf{Categorical:} &\x_d\sim \operatorname{Cat}(\pi_1, \pi_2, \dots, \pi_K).
\end{align*}

When it comes to evaluation we use missing imputation error, that is, for the imputed missing values that are numerical we compute the normalized root mean squared error (NRMSE),
\begin{equation}
    err(d) = \frac{1}{N}\frac{\norm[2]{\x_d - \hat{\x}_d}}{\max{(\x_d)} - \min{(\x_d)}},
\end{equation}
where $\hat{\x}$ is the value inferred by the model, and in the case of nominal data we compute the error rate, i.e.,
\begin{equation}
    err(d) = \frac{1}{N} \sum_{n=1}^N I(\x_{n,d} \neq \hat{\x}_{n,d}).
\end{equation}
The final metric is the mean across dimensions, $err = \frac{1}{D} \sum_d err(d)$.

\section{Additional experimental results} \label{appendix:E}

\begin{figure}
	\centering
	\subcaptionbox*{}{\includegraphics[width=\linewidth, keepaspectratio]{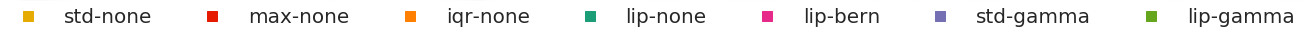}} \\
	\vspace{-\baselineskip}
	\subcaptionbox*{}{\includegraphics[width=.32\linewidth, keepaspectratio]{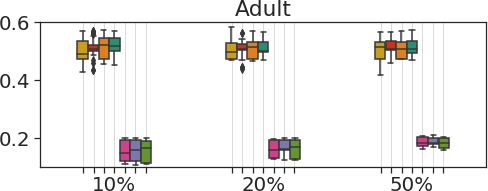}} \hfill %
	\subcaptionbox*{}{\includegraphics[width=.32\linewidth, keepaspectratio]{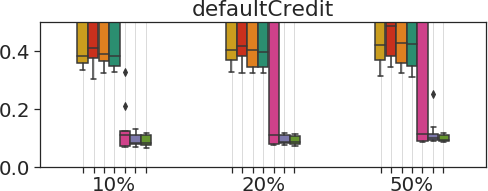}} \hfill %
	\subcaptionbox*{}{\includegraphics[width=.32\linewidth, keepaspectratio]{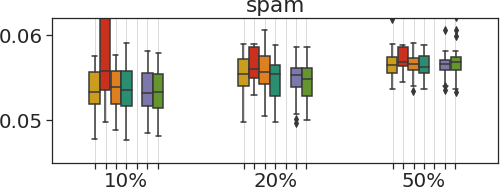}} \\ %
	\vspace{-\baselineskip}
	\subcaptionbox*{}{\includegraphics[width=.32\linewidth, keepaspectratio]{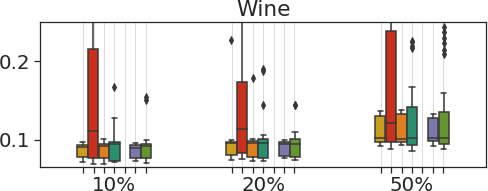}} \hfill %
	\subcaptionbox*{}{\includegraphics[width=.32\linewidth, keepaspectratio]{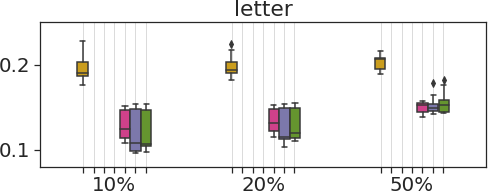}} \hfill %
	\subcaptionbox*{}{\includegraphics[width=.32\linewidth, keepaspectratio]{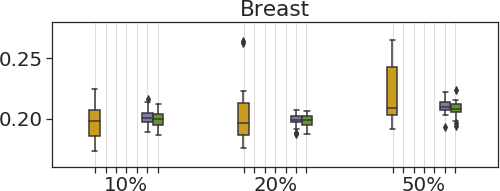}} %
	\vspace{-\baselineskip}
	\caption{Missing imputation error across different datasets and missing-values percentages. Lower is better.} \label{fig:boxplots-2}
\end{figure} %

\begin{figure}
    \centering
    \includegraphics[width=\textwidth, keepaspectratio]{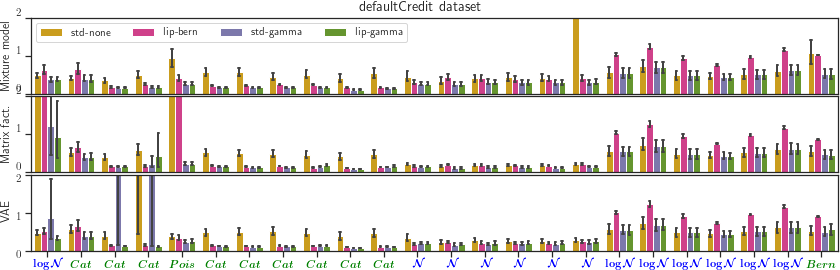}
    \caption{Per-dimension normalized missing imputation error on the \emph{defaultCredit} dataset (lower is better).}
    \label{fig:barplot-default}
\end{figure}

\begin{figure}
    \centering
    \includegraphics[width=\textwidth, keepaspectratio]{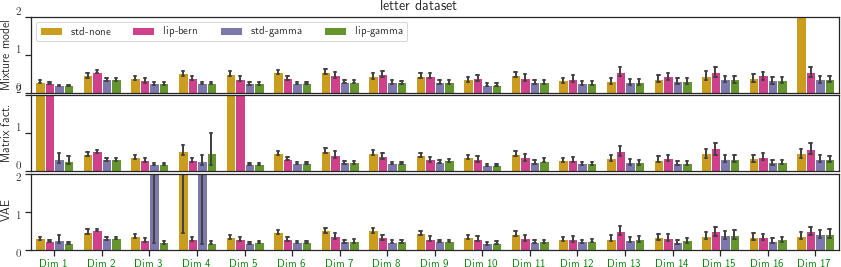}
    \caption{Per-dimension normalized missing imputation error on the \emph{letter} dataset (lower is better).}
    \label{fig:barplot-letter}
\end{figure}

In this section we show complementary results from the experiments performed in the main paper. First, Figure~\ref{fig:boxplots-2} depicts the same data as Figure~\ref{fig:boxplots} of the main paper, but averaging across models instead of missing-values percentages. 
Second, we plot in Figures~\ref{fig:barplot-default} and~\ref{fig:barplot-letter} per-dimension barplots of the normalized missing imputation error as in Figure~\ref{fig:stars-hetero}, now for the \emph{defaultCredit} and \emph{letter} datasets, respectively. These figures further validate the argument of \lipgamma not overlooking any variable, unlike \lipbern and \stdgamma.
Finally, we present the results in tabular form, divided by type of variable (discrete vs. continuous) and type of model (mixture model, matrix factorization and VAE). Tables \ref{tab:10percent}, \ref{tab:20percent}, and \ref{tab:50percent} show the results obtained with a \SI{10}{\percent}, \SI{20}{\percent}, and \SI{50}{\percent} of missing values, respectively. Major differences have been colored to ease their reading.

As discussed in Section \ref{sec:exps}, applying Lipschitz standardization results in an improvement on the imputation error across all datasets, being in the worst case as good as the best of the other methods. 
We can also observe how this improvement mainly manifests on discrete random variables when the Bernoulli and Gamma tricks are applied, and that the effect of data scaling is less noticeable as the expressiveness of the model increases. 
There are cases, like in the \emph{Adult} dataset, where there is a trade-off on learning the discrete dimensions and worsening the results on continuous dimensions.
However, the case where properly learning the discrete distributions translates to an improvement on all dimensions can also occur, as in the \emph{defaultCredit} dataset.

Finally, there is an important aspect that qualitatively differentiates \lipgamma from \lipbern and \stdgamma. The consequence of Lipschitz standardizing every dimension is obtaining the more balanced learning that we aim for, and in cases with high heterogeneity, such as \emph{defaultCredit} and \emph{Adult}, the stability and robustness of the algorithm increases. A clear example of this can be seen by checking the evolution of the \emph{defaultCredit} dataset on Tables~\ref{tab:10percent}, \ref{tab:20percent}, and \ref{tab:50percent}.
\textit{It is worth-noting that \lipgamma keeps achieving consistent results even under a half missing-data regime, which is impressive.}

\sisetup{omit-uncertainty=false}

\newpage

\begin{table}[!hbtp]
    \centering
    \caption{Missing imputation error with a \SI{10}{\percent} of missing data.} \label{tab:10percent}
    \resizebox{\textwidth}{!}{
        \begin{tabular}{llcccccc} \toprule
        \multicolumn{2}{l}{} & \multicolumn{3}{c}{Discrete} & \multicolumn{3}{c}{Continuous} \\
        \multicolumn{2}{c}{\textit{Imputation error}}  & \textit{Mixture} & \textit{Matrix fact.} & \textit{VAE} & \textit{Mixture} & \textit{Matrix fact.} & \textit{VAE}\\ \midrule
		\multirow{7}{*}{defaultCredit} & \stdnone & \num{0.770 (028)} & \red{\num{4.448 (7)}} & \num{0.712 (024)} & \num{0.055 (001)} & $\red{\infty}$ & \num{0.042 (003)} \\
 & \maxnone & \num{0.773 (022)} & $\red{\infty}$ & \num{0.720 (055)} & \num{0.134 (051)} & \num{0.056 (002)} & \num{0.038 (002)} \\
 & \iqrnone & \num{0.777 (025)} & \red{\num{8.486 (22)}} & \num{0.719 (036)} & \num{0.058 (009)} & $\red{\infty}$ & \num{0.044 (010)} \\
 & \lipnone & \num{0.775 (019)} & \num{0.803 (115)} & \num{0.705 (036)} & \num{0.054 (001)} & $\red{\infty}$ & \num{0.040 (001)} \\
 & \lipbern & \green{\num{0.195 (004)}} & \green{\num{0.133 (001)}} & \green{\num{0.123 (002)}} & \green{\num{0.044 (002)}} & $\red{\infty}$ & \green{\num{0.030 (001)}} \\
 & \stdgamma & \green{\num{0.189 (005)}} & \green{\num{0.143 (003)}} & \green{\num{0.123 (006)}} & \green{\num{0.045 (001)}} & \num{0.043 (032)} & \red{\num{0.126 (280)}} \\
 & \lipgamma & \green{\num{0.189 (005)}} & \num{0.144 (002)} & \green{\num{0.117 (009)}} & \green{\num{0.045 (001)}} & \red{\num{0.118 (251)}} & \green{\num{0.033 (002)}} \\ \midrule
\multirow{7}{*}{Adult} & \stdnone & \num{0.600 (002)} & \num{0.622 (052)} & \num{0.706 (022)} & \num{0.087 (001)} & \num{0.081 (001)} & \num{0.071 (002)} \\
 & \maxnone & \num{0.645 (003)} & \num{0.618 (051)} & \num{0.694 (037)} & \num{0.089 (000)} & \num{0.089 (000)} & \num{0.078 (005)} \\
 & \iqrnone & \num{0.601 (004)} & \num{0.671 (038)} & \num{0.702 (036)} & \num{0.087 (001)} & \num{0.081 (001)} & \num{0.072 (003)} \\
 & \lipnone & \num{0.639 (006)} & \num{0.651 (047)} & \num{0.713 (019)} & \num{0.088 (001)} & \num{0.082 (003)} & \num{0.071 (002)} \\
 & \lipbern & \green{\num{0.231 (004)}} & \green{\num{0.168 (002)}} & \green{\num{0.130 (005)}} & \num{0.087 (003)} & \num{0.094 (003)} & \num{0.073 (005)} \\
 & \stdgamma & \green{\num{0.229 (004)}} & \green{\num{0.182 (003)}} & \green{\num{0.125 (003)}} & \num{0.087 (003)} & \num{0.087 (003)} & \red{\num{0.503 (1)}} \\
 & \lipgamma & \green{\num{0.228 (004)}} & \green{\num{0.188 (006)}} & \green{\num{0.127 (015)}} & \num{0.087 (003)} & \num{0.097 (008)} & \num{0.085 (007)} \\ \midrule
\multirow{6}{*}{Wine} & \stdnone & \num{0.099 (005)} & \num{0.090 (002)} & \num{0.089 (008)} & \num{0.093 (001)} & \num{0.198 (337)} & \num{0.073 (002)} \\
 & \maxnone & \num{0.110 (007)} & \red{\num{0.352 (110)}} & \num{0.114 (063)} & \num{0.111 (001)} & \num{0.274 (075)} & \num{0.069 (000)} \\
 & \iqrnone & \num{0.099 (005)} & \num{0.092 (002)} & \num{0.086 (008)} & \num{0.093 (001)} & \num{0.148 (170)} & \num{0.071 (003)} \\
 & \lipnone & \num{0.099 (004)} & \num{0.097 (005)} & \num{0.089 (007)} & \num{0.093 (001)} & \num{0.287 (534)} & \num{0.069 (001)} \\
 & \stdgamma & \num{0.099 (003)} & \num{0.090 (002)} & \num{0.087 (005)} & \num{0.092 (001)} & \num{0.208 (376)} & \num{0.073 (001)} \\
 & \lipgamma & \num{0.100 (004)} & \num{0.092 (003)} & \num{0.088 (008)} & \num{0.093 (001)} & \red{\num{0.476 (1)}} & \num{0.071 (003)} \\ \midrule
\multirow{6}{*}{spam} & \stdnone & \num{0.144 (021)} & \num{0.080 (007)} & \num{0.094 (012)} & \num{0.054 (001)} & \num{0.054 (001)} & \num{0.050 (002)} \\
 & \maxnone & \num{0.158 (018)} & \num{0.081 (012)} & \red{\num{0.232 (122)}} & \num{0.054 (001)} & \num{0.054 (001)} & $\red{\infty}$ \\
 & \iqrnone & \num{0.149 (022)} & \num{0.081 (007)} & \num{0.086 (016)} & \num{0.054 (001)} & \num{0.054 (001)} & $\red{\infty}$ \\
 & \lipnone & \num{0.143 (022)} & \num{0.082 (006)} & \num{0.085 (010)} & \num{0.054 (001)} & \num{0.054 (001)} & \num{0.050 (003)} \\
 & \stdgamma & \num{0.167 (033)} & \num{0.082 (008)} & \num{0.090 (010)} & \num{0.054 (001)} & \num{0.054 (001)} & \num{0.050 (001)} \\
 & \lipgamma & \num{0.165 (035)} & \num{0.082 (008)} & \num{0.088 (015)} & \num{0.054 (001)} & \num{0.054 (001)} & \num{0.050 (001)} \\ \midrule
\multirow{4}{*}{Letter} & \stdnone & \num{0.210 (008)} & \num{0.190 (001)} & \num{0.183 (005)} & \minusmark & \minusmark & \minusmark \\
 & \lipbern & \green{\num{0.149 (002)}} & \num{0.125 (000)} & \num{0.112 (002)} & \minusmark & \minusmark & \minusmark \\
 & \stdgamma & \green{\num{0.150 (002)}} & \green{\num{0.108 (001)}} & \green{\num{0.098 (000)}} & \minusmark & \minusmark & \minusmark \\
 & \lipgamma & \green{\num{0.149 (002)}} & \green{\num{0.106 (001)}} & \green{\num{0.103 (003)}} & \minusmark & \minusmark & \minusmark \\ \midrule
\multirow{3}{*}{Breast} & \stdnone & \num{0.198 (005)} & \num{0.212 (006)} & \green{\num{0.183 (006)}} & \minusmark & \minusmark & \minusmark \\
 & \stdgamma & \num{0.201 (005)} & \num{0.200 (007)} & \num{0.201 (007)} & \minusmark & \minusmark & \minusmark \\
 & \lipgamma & \num{0.200 (005)} & \num{0.199 (006)} & \num{0.200 (007)} & \minusmark & \minusmark & \minusmark \\ \bottomrule
        \end{tabular}
    }
\end{table}

\newpage
\begin{table}[!hbtp]
    \centering
    \caption{Missing imputation error with a \SI{20}{\percent} of missing data.} \label{tab:20percent}
    \resizebox{\textwidth}{!}{
        \begin{tabular}{llcccccc} \toprule
        \multicolumn{2}{l}{} & \multicolumn{3}{c}{Discrete} & \multicolumn{3}{c}{Continuous} \\
        \multicolumn{2}{r}{\textit{Imputation error}}  & \textit{Mixture} & \textit{Matrix fact.} & \textit{VAE} & \textit{Mixture} & \textit{Matrix fact.} & \textit{VAE}\\ \midrule
        \multirow{7}{*}{defaultCredit} & \stdnone & \num{0.805 (023)} & $\red{\infty}$ & \num{0.707 (034)} & \num{0.055 (001)} & $\red{\infty}$ & \num{0.046 (007)} \\
 & \maxnone & \num{0.805 (018)} & $\red{\infty}$ & \num{0.739 (047)} & \num{0.110 (015)} & \num{0.056 (003)} & \num{0.038 (002)} \\
 & \iqrnone & \num{0.803 (021)} & \red{\num{9.938 (15)}} & \num{0.689 (016)} & \num{0.054 (001)} & $\red{\infty}$ & \num{0.042 (002)} \\
 & \lipnone & \num{0.807 (017)} & \red{\num{3.957 (4)}} & \num{0.686 (018)} & \num{0.053 (001)} & $\red{\infty}$ & \num{0.042 (003)} \\
 & \lipbern & \green{\num{0.192 (002)}} & \red{\num{0.400 (461)}} & \green{\num{0.133 (001)}} & \green{\num{0.044 (001)}} & $\red{\infty}$ & \green{\num{0.030 (001)}} \\
 & \stdgamma & \green{\num{0.186 (004)}} & \green{\num{0.146 (002)}} & \green{\num{0.133 (007)}} & \green{\num{0.045 (001)}} & \green{\num{0.039 (019)}} & \green{\num{0.037 (003)}} \\
 & \lipgamma & \green{\num{0.185 (003)}} & \green{\num{0.147 (001)}} & \green{\num{0.124 (002)}} & \green{\num{0.046 (001)}} & \green{\num{0.036 (007)}} & \green{\num{0.036 (003)}} \\ \midrule
\multirow{7}{*}{Adult} & \stdnone & \num{0.602 (004)} & \num{0.633 (023)} & \num{0.701 (024)} & \num{0.089 (001)} & \num{0.084 (002)} & \num{0.082 (034)} \\
 & \maxnone & \num{0.644 (002)} & \num{0.630 (054)} & \num{0.671 (040)} & \num{0.090 (001)} & \num{0.090 (001)} & \num{0.073 (001)} \\
 & \iqrnone & \num{0.601 (003)} & \num{0.656 (030)} & \num{0.702 (024)} & \num{0.089 (001)} & \num{0.084 (003)} & \num{0.071 (001)} \\
 & \lipnone & \num{0.634 (004)} & \num{0.648 (030)} & \num{0.686 (035)} & \num{0.090 (001)} & \num{0.083 (001)} & \num{0.072 (002)} \\
 & \lipbern & \green{\num{0.231 (002)}} & \green{\num{0.180 (004)}} & \green{\num{0.146 (001)}} & \num{0.087 (002)} & \num{0.094 (002)} & \num{0.077 (003)} \\
 & \stdgamma & \green{\num{0.230 (003)}} & \green{\num{0.188 (005)}} & \num{0.163 (033)} & \num{0.087 (002)} & \num{0.087 (002)} & $\red{\infty}$ \\
 & \lipgamma & \green{\num{0.230 (002)}} & \green{\num{0.195 (007)}} & \green{\num{0.141 (002)}} & \num{0.087 (002)} & \num{0.096 (007)} & \num{0.084 (004)} \\ \midrule
\multirow{6}{*}{Wine} & \stdnone & \num{0.107 (007)} & \num{0.099 (001)} & \num{0.089 (002)} & \num{0.094 (001)} & \num{0.113 (048)} & \num{0.076 (002)} \\
 & \maxnone & \num{0.118 (009)} & \red{\num{0.281 (120)}} & \red{\num{0.125 (048)}} & \num{0.112 (000)} & \red{\num{0.235 (069)}} & \num{0.073 (000)} \\
 & \iqrnone & \num{0.105 (006)} & \num{0.101 (002)} & \num{0.087 (004)} & \num{0.094 (001)} & \num{0.109 (029)} & \num{0.074 (001)} \\
 & \lipnone & \num{0.106 (006)} & \num{0.103 (006)} & \num{0.090 (007)} & \num{0.094 (001)} & \num{0.159 (101)} & \num{0.073 (001)} \\
 & \stdgamma & \num{0.101 (006)} & \num{0.099 (002)} & \num{0.092 (004)} & \num{0.093 (001)} & \num{0.121 (078)} & \num{0.076 (002)} \\
 & \lipgamma & \num{0.103 (006)} & \num{0.099 (002)} & \num{0.094 (006)} & \num{0.094 (001)} & \red{\num{0.240 (394)}} & \num{0.073 (001)} \\ \midrule
\multirow{6}{*}{spam} & \stdnone & \num{0.186 (035)} & \num{0.088 (012)} & \num{0.094 (007)} & \num{0.055 (001)} & \num{0.055 (001)} & \num{0.060 (018)} \\
 & \maxnone & \num{0.176 (025)} & \num{0.089 (014)} & \red{\num{0.222 (097)}} & \num{0.055 (001)} & \num{0.055 (001)} & $\red{\infty}$ \\
 & \iqrnone & \num{0.185 (034)} & \num{0.086 (012)} & \num{0.093 (011)} & \num{0.055 (001)} & \num{0.055 (001)} & $\red{\infty}$ \\
 & \lipnone & \num{0.180 (033)} & \num{0.087 (009)} & \num{0.098 (012)} & \num{0.055 (001)} & \num{0.055 (001)} & \green{\num{0.052 (004)}} \\
 & \stdgamma & \green{\num{0.168 (022)}} & \num{0.099 (009)} & \num{0.100 (011)} & \num{0.055 (001)} & \num{0.055 (001)} & $\red{\infty}$ \\
 & \lipgamma & \green{\num{0.169 (030)}} & \num{0.095 (008)} & \num{0.096 (008)} & \num{0.055 (001)} & \num{0.055 (001)} & \green{\num{0.051 (001)}} \\ \midrule
\multirow{4}{*}{Letter} & \stdnone & \red{\num{0.210 (007)}} & \red{\num{0.193 (000)}} & \red{\num{0.188 (004)}} & \minusmark & \minusmark & \minusmark \\
 & \lipbern & \green{\num{0.150 (001)}} & \green{\num{0.131 (000)}} & \green{\num{0.120 (002)}} & \minusmark & \minusmark & \minusmark \\
 & \stdgamma & \green{\num{0.151 (001)}} & \green{\num{0.114 (001)}} & \green{\num{0.111 (003)}} & \minusmark & \minusmark & \minusmark \\
 & \lipgamma & \green{\num{0.151 (001)}} & \green{\num{0.112 (001)}} & \green{\num{0.120 (003)}} & \minusmark & \minusmark & \minusmark \\ \midrule
\multirow{3}{*}{Breast} & \stdnone & \num{0.196 (004)} & \num{0.224 (021)} & \num{0.183 (004)} & \minusmark & \minusmark & \minusmark \\
 & \stdgamma & \num{0.196 (006)} & \num{0.200 (002)} & \num{0.201 (004)} & \minusmark & \minusmark & \minusmark \\
 & \lipgamma & \num{0.197 (005)} & \num{0.200 (002)} & \num{0.198 (006)} & \minusmark & \minusmark & \minusmark \\ \bottomrule
        \end{tabular}
    }
\end{table}

\newpage
\begin{table}[!hbtp]
    \centering
    \caption{Missing imputation error with a \SI{50}{\percent} of missing data.} \label{tab:50percent}
    \resizebox{\textwidth}{!}{
        \begin{tabular}{llcccccc} \toprule
        \multicolumn{2}{l}{} & \multicolumn{3}{c}{Discrete} & \multicolumn{3}{c}{Continuous} \\
        \multicolumn{2}{r}{\textit{Imputation error}}  & \textit{Mixture} & \textit{Matrix fact.} & \textit{VAE} & \textit{Mixture} & \textit{Matrix fact.} & \textit{VAE}\\ \midrule
        \multirow{7}{*}{defaultCredit} & \stdnone & \num{0.829 (037)} & $\red{\infty}$ & \num{0.709 (045)} & $\red{\infty}$ & $\red{\infty}$ & \num{0.046 (003)} \\
 & \maxnone & \num{0.833 (025)} & $\red{\infty}$ & \num{0.764 (051)} & $\red{\infty}$ & \num{0.057 (001)} & \num{0.045 (004)} \\
 & \iqrnone & \num{0.831 (024)} & $\red{\infty}$ & \num{0.709 (028)} & $\red{\infty}$ & $\red{\infty}$ & \num{0.045 (003)} \\
 & \lipnone & \num{0.838 (041)} & $\red{\infty}$ & \num{0.690 (031)} & $\red{\infty}$ & $\red{\infty}$ & \num{0.044 (002)} \\
 & \lipbern & \green{\num{0.194 (002)}} & $\red{\infty}$ & \green{\num{0.154 (001)}} & \num{0.044 (001)} & $\red{\infty}$ & \green{\num{0.033 (000)}} \\
 & \stdgamma & \green{\num{0.191 (003)}} & \green{\num{0.160 (002)}} & \green{\num{0.163 (007)}} & \num{0.046 (001)} & \red{\num{0.165 (295)}} & \num{0.037 (002)} \\
 & \lipgamma & \green{\num{0.191 (003)}} & \green{\num{0.161 (002)}} & \green{\num{0.154 (009)}} & \num{0.046 (001)} & \green{\num{0.040 (012)}} & \num{0.037 (002)} \\ \midrule
\multirow{7}{*}{Adult} & \stdnone & \num{0.600 (001)} & \num{0.667 (052)} & \num{0.666 (057)} & \num{0.088 (000)} & \num{0.086 (002)} & \num{0.071 (003)} \\
 & \maxnone & \num{0.642 (001)} & \num{0.654 (048)} & \num{0.681 (041)} & \num{0.089 (000)} & \num{0.089 (000)} & \num{0.075 (003)} \\
 & \iqrnone & \num{0.600 (001)} & \num{0.668 (033)} & \num{0.685 (039)} & \num{0.088 (000)} & \num{0.088 (006)} & \num{0.072 (004)} \\
 & \lipnone & \num{0.633 (005)} & \num{0.675 (030)} & \num{0.666 (052)} & \num{0.089 (000)} & \num{0.085 (002)} & \num{0.072 (001)} \\
 & \lipbern & \green{\num{0.242 (002)}} & \green{\num{0.210 (001)}} & \green{\num{0.197 (005)}} & \num{0.087 (001)} & \num{0.096 (001)} & \num{0.084 (002)} \\
 & \stdgamma & \green{\num{0.239 (002)}} & \green{\num{0.209 (001)}} & \num{0.210 (017)} & \num{0.087 (000)} & \num{0.090 (008)} & \num{0.098 (009)} \\
 & \lipgamma & \green{\num{0.239 (002)}} & \green{\num{0.212 (001)}} & \green{\num{0.191 (003)}} & \num{0.087 (000)} & \num{0.097 (003)} & \num{0.081 (003)} \\ \midrule
\multirow{6}{*}{Wine} & \stdnone & \num{0.122 (005)} & \num{0.145 (004)} & \num{0.118 (010)} & \num{0.098 (002)} & \num{0.131 (002)} & \num{0.092 (003)} \\
 & \maxnone & \num{0.155 (020)} & \red{\num{0.264 (102)}} & \num{0.131 (020)} & \num{0.116 (001)} & \num{0.273 (038)} & \num{0.087 (001)} \\
 & \iqrnone & \num{0.122 (005)} & \num{0.148 (005)} & \num{0.116 (007)} & \num{0.098 (002)} & \num{0.132 (002)} & \num{0.091 (002)} \\
 & \lipnone & \num{0.122 (006)} & \num{0.159 (014)} & \num{0.113 (009)} & \num{0.098 (002)} & \num{0.181 (043)} & \num{0.088 (002)} \\
 & \stdgamma & \num{0.121 (006)} & \num{0.134 (004)} & \num{0.121 (006)} & \num{0.097 (001)} & \num{0.129 (001)} & \num{0.091 (002)} \\
 & \lipgamma & \num{0.121 (005)} & \num{0.140 (007)} & \num{0.110 (004)} & \num{0.098 (001)} & \num{0.201 (053)} & \num{0.089 (002)} \\ \midrule
\multirow{6}{*}{spam} & \stdnone & \num{0.188 (027)} & \num{0.118 (004)} & \num{0.144 (011)} & \num{0.055 (000)} & \num{0.055 (000)} & $\red{\infty}$ \\
 & \maxnone & \num{0.183 (018)} & \num{0.127 (003)} & \red{\num{0.328 (132)}} & \num{0.055 (000)} & \num{0.055 (000)} & $\red{\infty}$ \\
 & \iqrnone & \num{0.187 (027)} & \num{0.118 (003)} & \num{0.149 (017)} & \num{0.055 (000)} & \num{0.055 (000)} & $\red{\infty}$ \\
 & \lipnone & \num{0.188 (029)} & \num{0.122 (005)} & \num{0.147 (011)} & \num{0.055 (000)} & \num{0.056 (002)} & \num{0.053 (000)} \\
 & \stdgamma & \num{0.195 (048)} & \num{0.129 (006)} & \num{0.149 (013)} & \num{0.055 (000)} & \num{0.055 (000)} & $\red{\infty}$ \\
 & \lipgamma & \num{0.192 (049)} & \num{0.130 (007)} & \num{0.149 (027)} & \num{0.055 (000)} & \num{0.056 (001)} & \num{0.053 (000)} \\ \midrule
\multirow{4}{*}{Letter} & \stdnone & \num{0.210 (004)} & \num{0.207 (000)} & \num{0.192 (002)} & \minusmark & \minusmark & \minusmark \\
 & \lipbern & \green{\num{0.153 (001)}} & \num{0.155 (001)} & \green{\num{0.143 (002)}} & \minusmark & \minusmark & \minusmark \\
 & \stdgamma & \green{\num{0.154 (001)}} & \green{\num{0.145 (000)}} & \num{0.154 (010)} & \minusmark & \minusmark & \minusmark \\
 & \lipgamma & \green{\num{0.153 (001)}} & \green{\num{0.144 (000)}} & \num{0.165 (008)} & \minusmark & \minusmark & \minusmark \\ \midrule
\multirow{3}{*}{Breast} & \stdnone & \num{0.207 (004)} & \red{\num{0.251 (008)}} & \num{0.201 (005)} & \minusmark & \minusmark & \minusmark \\
 & \stdgamma & \num{0.208 (007)} & \num{0.210 (004)} & \num{0.213 (005)} & \minusmark & \minusmark & \minusmark \\
 & \lipgamma & \num{0.209 (006)} & \num{0.211 (004)} & \num{0.206 (006)} & \minusmark & \minusmark & \minusmark \\ \bottomrule
        \end{tabular}
    }
\end{table}

\sisetup{omit-uncertainty=false}

\end{document}